\documentclass[10pt]{article} 

\usepackage[preprint]{tmlr}

\usepackage[utf8]{inputenc} 
\usepackage[T1]{fontenc}    
\usepackage{hyperref}       
\usepackage{url}            
\usepackage{booktabs}       
\usepackage{amsfonts}       
\usepackage{nicefrac}       
\usepackage{microtype}      
\usepackage{xcolor}         
\usepackage{array,multirow,graphicx}
\usepackage{amsthm}
\usepackage{booktabs}
\usepackage{amsmath,amssymb}
\usepackage[ruled,vlined]{algorithm2e}
\usepackage{thm-restate,enumitem}
\usepackage{enumitem,kantlipsum}
\usepackage{xspace}
\usepackage{graphicx}
\usepackage{amsfonts}
\usepackage{amsthm}
\usepackage{amsmath}
\usepackage{mathtools}
\usepackage{amssymb}
\usepackage{etoolbox}
\usepackage{tabto}
\usepackage{bbm}
\usepackage{tikz}
\usetikzlibrary{positioning}

\tikzset{%
  every neuron/.style={
    circle,
    draw,
    minimum size=1cm
  },
  neuron missing/.style={
    draw=none, 
    scale=4,
    text height=0.333cm,
    execute at begin node=\color{black}$\vdots$
  },
}

\DeclareFontFamily{U}{dutchcal}{\skewchar\font=45}
\DeclareFontShape{U}{dutchcal}{m}{n}{<-> s*[1.0] dutchcal-r}{}
\DeclareFontShape{U}{dutchcal}{b}{n}{<-> s*[1.0] dutchcal-b}{}
\DeclareMathAlphabet{\mathlcal}{U}{dutchcal}{m}{n}
\SetMathAlphabet{\mathlcal}{bold}{U}{dutchcal}{b}{n}
\usepackage{mathrsfs}
\usepackage[scr=esstix]{mathalfa}

\def\smallint{\begingroup\textstyle \int\endgroup}
\DeclarePairedDelimiter{\abs}{\lvert}{\rvert}
\DeclarePairedDelimiter{\brk}{[}{]}
\DeclarePairedDelimiter{\crl}{\{}{\}}
\DeclarePairedDelimiter{\prn}{(}{)}
\DeclarePairedDelimiter{\norm}{\|}{\|}
\DeclarePairedDelimiter{\tri}{\langle}{\rangle}
\DeclarePairedDelimiter{\ceil}{\lceil}{\rceil}

\DeclareMathOperator{\sinc}{sinc}

\DeclareMathOperator\supp{supp}

\newcommand{\cF}{\mathcal{F}}

\newcommand{\cH}{\mathcal{H}}

\newcommand{\cO}{\mathcal{O}}

\newcommand{\R}{\mathbb{R}}
\newcommand{\C}{\mathbb{C}}
\newcommand{\Z}{\mathbb{Z}}

\DeclareMathOperator*{\argmin}{arg\,min}

\newcommand{\sob}[1]{\mathcal{W}^{r, #1}}
\newcommand{\nn}[1]{f^{\textnormal{#1}}}

\newcommand{\normalize}{\frac{1}{\prn{2\pi}^d}}

\def\fr#1#2{{\textstyle\frac{#1}{#2}}}

\newtheorem{theorem}{Theorem}
\newtheorem{lemma}{Lemma}

\newtheorem{definition}{Definition}

\newtheorem{example}{Example}
\definecolor{mycolor}{RGB}{153,153,255}
\definecolor{mycolor_orange}{RGB}{255,152,49}

\title{Exploring the Approximation Capabilities of Multiplicative Neural Networks for Smooth Functions}

\author{\name Ido Ben-Shaul \email ido.benshaul@gmail.com \\
      \addr Department of Applied Mathematics\\
      Tel-Aviv University, Israel \\
      eBay Research
      \AND
      \name Tomer Galanti \email tomer22g@gmail.com \\
      \addr Massachusetts Institute of Technology (MIT) \\ Cambridge, MA, USA \\
      \AND
      \name Shai Dekel \email shaidekel6@gmail.com\\
      \addr Department of Applied Mathematics\\
      Tel-Aviv University, Israel}
      
\begin{document}

\maketitle

\begin{abstract}
Multiplication layers are a key component in various influential neural network modules, including self-attention and hypernetwork layers. In this paper, we investigate the approximation capabilities of deep neural networks with intermediate neurons connected by simple multiplication operations. We consider two classes of target functions: generalized bandlimited functions, which are frequently used to model real-world signals with finite bandwidth, and Sobolev-Type balls, which are embedded in the Sobolev Space $\mathcal{W}^{r,2}$. Our results demonstrate that multiplicative neural networks can approximate these functions with significantly fewer layers and neurons compared to standard ReLU neural networks, with respect to both input dimension and approximation error. These findings suggest that multiplicative gates can outperform standard feed-forward layers and have potential for improving neural network design.
\end{abstract}

\section{Introduction}\label{sec:intro}

Deep learning with large neural networks has seen tremendous success in solving a wide range of tasks in recent years, including image classification~\citep{he2016residual,dosovitskiy2021an,zhai2021scaling}, language processing~\citep{NIPS2017_3f5ee243,devlin-etal-2019-bert,NEURIPS2020_1457c0d6}, interacting with open-ended environments~\citep{SilverHuangEtAl16nature,arulkumaran2019alphastar}, and code synthesis~\citep{chen2021evaluating}. 

Recent empirical studies have shown that neural networks that incorporate multiplication operations between intermediate neurons~\citep{10.1162/neco.1989.1.1.133,https://doi.org/10.48550/arxiv.1503.05724,https://doi.org/10.48550/arxiv.1808.00508}, such as self-attention layers~\citep{NIPS2017_3f5ee243} and hypernetworks~\citep{ha2016hypernetworks,krueger2017bayes,Littwin_2019_ICCV}, are particularly effective. For example, self-attention layers have been widely successful in computer vision~\citep{dosovitskiy2021an,zhai2021scaling} and language processing~\citep{https://doi.org/10.48550/arxiv.1601.06733,parikh-etal-2016-decomposable,https://doi.org/10.48550/arxiv.1705.04304,NIPS2017_3f5ee243}. It has also been shown that one can achieve reasonable performance with Transformers even without applying non-linear activation functions~\citep{Levine2020LimitsTD}. Additionally, hypernetworks, which use multiplication to generate network weights seem to improve the performance of neural networks on various meta-learning tasks~\citep{Oswald2020Continual,Littwin_2019_ICCV,NEURIPS2021_ac796a52}. However, the benefits of multiplication layers are not well understood from a theoretical perspective.

In this work, we study the expressive power of neural networks with multiplication layers. Specifically, we want to evaluate the number of neurons and layers needed to approximate a given function within a given error tolerance using a specific architecture. A classic result in the theory of deep learning shows that neural networks can approximate any smooth target function, known as the universal approximation property, with as few as one hidden layer~\citep{Cybenko1989ApproximationBS,Hornik1989MultilayerFN,Funahashi1989OnTA,Leshno1991MultilayerFN}. However, these papers do not provide specific information about the type of architecture and number of parameters required to achieve a given level of accuracy. This is a crucial question, as a high requirement for these resources could limit the universality of neural networks and explain their limited success in some practical applications. 

Previous work has demonstrated that functions in Sobolev spaces can be approximated by a one-hidden layer neural network with analytic activation functions~\citep{Mhaskar1996NeuralNF}. However, the number of neurons required to approximate these functions with an error of at most $\epsilon$ in the $L_{\infty}$ norm scales as $\mathcal{O}(\epsilon^{-d/r})$, where $d$ is the input dimension, $r$ is the smoothness degree of the target function, and $\epsilon>0$ is the error rate. This raises the question of whether the curse of dimensionality, the phenomenon whereby the complexity of a model grows exponentially with the input dimension, is inherent to neural networks.

On the other hand,~\cite{DeVore1989OptimalNA} proved that any continuous function approximator that approximates all Sobolev functions of order $r$ and dimension $d$ within error $\epsilon$ requires at least $\Omega(\epsilon^{-d/r})$ parameters in the $L_\infty$ norm. This result meets the bound of~\cite{Mhaskar1996NeuralNF} and confirms that neural networks cannot avoid the curse of dimensionality for the Sobolev space when approximating in the $L_{\infty}$ norm. A key question is whether neural networks can overcome this curse of dimensionality for certain sets of target functions, and what kind of architectures provide the best guarantees for approximating these functions. 

To overcome the curse of dimensionality, various studies~\citep{10.5555/3298483.3298577,doi:10.1073/pnas.1907369117,7762943,doi:10.1137/18M1189336,blanchard2022shallow,galanti2020modularity} have investigated the approximation capabilities of neural networks in representing other classes of functions with weaker notions of distance, such as the $L_2$ distance. For example,~\cite{10.5555/3298483.3298577,doi:10.1073/pnas.1907369117} showed that smooth, compositionally sparse functions with a degree of smoothness $r$ can be approximated with the $L_{\infty}$ distance up to error $\epsilon$ using deep neural networks with $\mathcal{O}(d\epsilon^{-2/r})$ neurons. Other structural constraints have been applied to functions with structured input spaces~\citep{MHASKAR201063,10.5555/3455716.3455890,https://doi.org/10.48550/arxiv.1908.00695}, compositions of functions~\citep{7762943}, piecewise smooth functions~\citep{Petersen2017OptimalAO, Imaizumi2018DeepNN}. A different line of research has focused on understanding the types of functions that certain neural network architectures can implement with regularity constraints. For example,~\cite{https://doi.org/10.48550/arxiv.1906.08039} showed that the space of 2-layer neural networks is equivalent to the Barron space when the size of their weights is restricted. They further showed that Barron functions can be approximated within $\epsilon$ using 2-layer networks with $\cO(\epsilon^{-2})$ neurons. Another line of research has considered spectral conditions on the function space, allowing functions to be expressed as infinite-width limits of shallow networks~\citep{10.5555/114836.114859,https://doi.org/10.48550/arxiv.1607.01434}. In~\citep{blanchard2022shallow} they considered the space of Korobov functions, which are functions that are practically useful for solving partial differential equations (PDEs). They showed any Korobov function can be approximated up to error $\epsilon$ in $L_2$ distance with a 2-layer neural network with ReLU activation function with $\mathcal{O}(\epsilon^{-1}\log(1/\epsilon)^{1.5(d-1)+1})$ and with a $\mathcal{O}(\log(d))$-depth network with $\mathcal{O}(\epsilon^{-0.5}\log(1/\epsilon)^{1.5(d-1)+1})$ neurons.

In a recent paper,~\cite{Montanelli2021DeepRN} provided approximation guarantees were established for generalized bandlimited functions. These functions are commonly used to model signals that have a finite range of frequencies (e.g., waves, video, and audio signals), which is known as a finite bandwidth. The solutions to many PDEs in physics are bandlimited functions, as the physical phenomena modeled by these PDEs typically have a finite range of frequencies. For example, the solutions to the wave equation, which models the propagation of waves, are bandlimited functions. In~\citep{Montanelli2021DeepRN}, it was shown that any bandlimited function can be approximated to within error $\epsilon$ using a ReLU neural network of depth $\cO(\log^2(1/\epsilon))$ with $\cO(\epsilon^{-2}\log^2(1/\epsilon))$ neurons with the $L_2$ distance.

In this paper, we study the approximation abilities of multiplicative neural network architectures with the $L_{2}$ distance. In particular, we prove that a multiplicative neural network of depth $\cO(\log(\fr{1}{\epsilon}))$ with $\cO(\epsilon^{-2}\log(\fr{1}{\epsilon}))$ neurons can approximate any generalized bandlimited function up to an error of $\epsilon$ (with constants depending on the dimension and on the band). Additionally, we also study the approximation guarantees of neural networks for approximating functions in Sobolev-Type balls of order $r$. We show that for the same error tolerance $\epsilon$, multiplicative neural networks can approximate these functions with depth $\cO(d^2\epsilon^{-1/r})$ and $\cO(d^2\epsilon^{-(1+1/r)})$ neurons, while standard ReLU neural networks require depth $\cO(d^2\epsilon^{-2/r})$ and $\cO(d^2\epsilon^{-(2+2/r)})$ neurons. These results demonstrate the superior performance of multiplicative gates compared to standard fully-connected layers. In Table~\ref{tab:contributions} we contrast our new bounds with preexisting bounds on the approximation power of neural networks for the Sobolev space, bandlimited functions, and the Sobolev-Type ball.

\begin{table}[h]
\centering
 \resizebox{\linewidth}{!}{
 \begin{tabular}{|ccccc|} 
 \hline
Space & Model & \# neurons & Depth & Reference  \\ 
 \hline
 $\mathcal{W}^{r, p}$ & $C^\infty$, non-poly & $\cO\prn{\epsilon ^ {-d/r}}$ & $\cO\prn{1}$ & \citep{Mhaskar1996NeuralNF} \\
 
 $\mathcal{W}^{r, \infty}$ & ReLU & $\cO\prn{\epsilon ^ {-d/r} \log {\frac{1}{\epsilon}}}$ & $\cO\prn{\log {\frac{1}{\epsilon}}}$ & \citep{Yarotsky2017ErrorBF} \\

 \hline
 Bandlimited functions & ReLU & $\cO\prn{\epsilon^{-2} \log^2 \frac{1}{\epsilon}}$ & $\cO\prn{\log^2{\frac{1}{\epsilon}}}$ & \citep{Montanelli2021DeepRN} \\
 
 Bandlimited functions & Multiplicative & $\cO\prn{\epsilon^{-2} \log\frac{1}{\epsilon}}$ & $\cO\prn{\log{\frac{1}{\epsilon}}}$ & This paper \\
 
 \hline
 $\mathcal{B}_{2r, 2}\subsetneq \mathcal{W}^{r, 2}$ & ReLU & $\cO\prn{d^2 \epsilon^{-\prn{2+2/r}}}$ & $\cO\prn{d^2 \epsilon^{-2/r}}$ & This paper \\
 
 $\mathcal{B}_{2r, 2}\subsetneq \mathcal{W}^{r, 2}$ & Multiplicative & $\cO\prn{d \epsilon^{-\prn{2+1/r}}}$ & $\cO\prn{d \epsilon^{-1/r}}$ & This paper \\
 
 \hline
 
 $\mathcal{B}_{1, 1}$ & Sigmodial & $\cO\prn{d \epsilon^{-2}}$ & $\cO\prn{1}$ & \citep{Barron1993UniversalAB} \\
 
 \hline
 
 \end{tabular}
 }
    \caption{Approximation results for Sobolev $\mathcal{W}^{r,p}$, Bandlimited and $\mathcal{B}_{2r, 2}$ functions by ReLU and multiplicative neural networks. The number of neurons and the depth are given in $\cO$ notation.}
  \label{tab:contributions}
\end{table}

\section{Problem Setup}\label{sec:setup}

We are interested in determining how complex (i.e., number of trainable parameters, number of neurons and layers) a model ought to be in order to theoretically guarantee approximation of an unknown target function $f$ up to a given approximation error $\epsilon > 0$. 

Formally, we consider a Banach space of functions $\mathcal{V}$ (for example, $L_{p}([0,1]^d)$), equipped with a norm $\norm{\cdot}_{\mathcal{V}}$ (for example, $\|\cdot \|_{L_p([0,1]^d)}$), and a set of target functions $\mathcal{U} \subseteq \mathcal{V}$. We also consider a set of approximators $\mathcal{H}$ and seek to quantify the ability of these approximators to approximate $\mathcal{U}$ using the following quantity 
\begin{equation*}
d_{\mathcal{V}}(\mathcal{H}, \mathcal{U}) ~=~ \inf_{\hat{f} \in \mathcal{H}} {\sup_{f \in \mathcal{U}} \|\hat{f} - f \|_{\mathcal{V}}},
\end{equation*}
which measures the maximal approximation error for approximating a target function $f \in \mathcal{U}$ using candidates $\hat{f}$ from $\mathcal{H}$. Typically, $\mathcal{H}$ is a parametric set of functions (e.g., neural networks of a certain architecture) and we denote by $\hat{f}_{\theta} \in \mathcal{H}$ a function that is parameterized by a vector of parameters $\theta \in \R^{N}$. For simplicity, we avoid writing $\theta$ in the subscript when it is obvious from context. 


\subsection{Target Function Spaces}

It is generally impossible to approximate arbitrary target functions using standard neural networks, as demonstrated in Theorem 7.2 in \citep{Devroye1996APT}. As a result, we often consider specific spaces of target functions that satisfy certain smoothness assumptions in order to obtain non-trivial results. In this work, we focus specifically on target functions $\mathcal{U}$ that satisfy the following smoothness assumptions. In this paper, we will focus on approximating functions on the unit cube $B=[0, 1]^d$, with $\abs{B}=1$, the volume of $B$.

{\bf Sobolev spaces.\enspace} Sobolev spaces are one of the most extensively studied classes of functions in approximation theory~\citep{DeVore1993ConstructiveA,Yarotsky2017ErrorBF,Liang2017WhyDN}. These spaces consist of smooth functions with bounded derivatives up to a certain order and are particularly useful in the study of partial differential equations (PDEs).

We first define the $L_p$ norm of a given function $f:\Omega \to \R$ as $\|f\|_{L_p(\Omega)}=(\smallint_{\Omega} \abs{ f(x)}^p~\textnormal{d}x)^{1/p}$, where $\Omega$ is a measurable space equipped with a Sigma-algebra $\Sigma$ and a measure $\mu$. A function $f$ is said to be in $L_p \prn{\Omega}$ if $\norm{f}_{L_p \prn{\Omega}} < \infty$.

Let $r\in \mathbb{N}$ and $p \in [1, \infty)$. The Sobolev space $\sob{p}\prn*{B}$ consists of functions $f:B \rightarrow \mathbb{R}$ with $r$-distributional derivatives in $L_p$. The Sobolev norm $\norm{\cdot}_{\sob{p}\prn{B}}$ is defined as
\[
\norm{f}_{\sob{p}\prn{B}} ~=~ \sum_{\alpha: \abs{\alpha}_1\leq r} \norm{D^{\alpha} f}_{L_p \prn{B}},
\]
where $\alpha=(\alpha_1,\dots,\alpha_d) \in \{0,\dots,r\}^d, \abs{\alpha}_1 = \alpha_1 + \dots + \alpha_d$, and $D^{\alpha} f$ is the respective distributional derivative. When $p=\infty$, the essential supremum norm is used. We also present the semi-norm:
\[
\abs{f}_{r, p} ~=~ \sum_{\alpha ~:~ \abs{\alpha} = r} \norm{D^{\alpha} f\prn{x}}_{L_p \prn{B}}.
\]


Classic results in the literature show that the number of parameters needed to approximate functions in $\sob{\infty}$ up to error $\epsilon$ is lower bounded by $\Omega(\epsilon^{-d/r})$ \citep{DeVore1989OptimalNA}. This exponential dependence on $d$ is known as the ``curse of dimensionality''.

{\bf Generalized Bandlimited functions.\enspace} Generalized bandlimited functions are functions whose spectrum, or the set of frequencies that make up the function, is limited to a certain band or range of frequencies. This property makes bandlimited functions well-suited for certain applications, such as signal processing, where it is important to ensure that the signal does not contain frequencies outside of a certain range. 

Generalized bandlimited functions are defined using their generalized Fourier transform given by an analytic function $K:\R\to\C$ (known as the kernel) and a scalar $M\geq 1$ (the band). Formally, for a given function $f:\R^d\to\R$, we denote the set of all functions $F:[-M, M]^d \to \C$ that retrieve $f$ from the frequencies $\omega \in [-M,M]^d$ using the kernel $K$ as follows:
\begin{equation*}
S_{f, K} ~=~ \left\{F:[-M,M]^d \to \R \mid f(x) = \smallint_{\brk{-M, M}^d} F\prn{\omega} K\prn{\omega\cdot x} ~\textnormal{d}\omega \right\}.
\end{equation*}
We define $F_f = \argmin_{F\in{S_{f, K}}} \norm{F}_{L^2 \prn{\brk{-M, M}^d}}$ to be the smallest $L_2$-norm function in $S_{f, K}$. For instance, when {$K(x)= \exp(ix)$, the function $F_f$ corresponds to the normalized standard Fourier transform of $f$, which is given by $F= \fr{1}{(2\pi)^d}\prn{\mathcal{F} f }\prn{\omega} = \fr{1}{(2\pi)^d}\int_{\R^d} f\prn{x} \exp\prn{- i\omega \cdot x} ~\textnormal{d}x$.} 


The space $\mathcal{H}_{K,M}(B)$ of generalized bandlimited functions is a Hilbert space of functions that can be represented as a weighted sum of the function $K$ over a finite domain. This space is equipped with an inner product and a norm, which allow us to measure the similarity and magnitude of these functions, respectively. We define $\cH_{K, M}\prn{B}$ as the functions $f:B\rightarrow\R$ such that
\begin{small}
\begin{equation*}
\cH_{K, M}\prn{B} ~=~ \crl*{\forall x\in B, f\prn{x} = \smallint_{\brk{-M, M}^d} { F\prn{\omega}K\prn{\omega \cdot x} ~\textnormal{d}\omega} \mid F:\brk{-M, M}^d \rightarrow \C \text{ is in } L^2 \prn{\brk{-M, M}^d}}.
\end{equation*}
\end{small}
The inner product and norm in this space are defined as follows:
$\langle f,g \rangle_{\cH_{K, M}\prn{B}} = \smallint_{\brk{-M, M}^d} F_f\prn{\omega} \overline{F_g}\prn{\omega}~\textnormal{d}\omega$ and norm $\norm{f}_{\cH_{K, M}\prn{B}} = \norm{F_f}_{L^2{\prn{\brk{-M, M}^d}}}$.

One of the key properties of generalized bandlimited functions is that they can be completely reconstructed from a discrete set of samples. This is known as the Shannon-Nyquist theorem~\citep{Shannon1984CommunicationIT}, and it is an important result in the field of signal processing and communication. An interesting consequence of this theorem is that even in high-dimensions, where seemingly unpredictable geometrical phenomena may occur (e.g. \cite{Blum2020FoundationsOD}, Chapter 2),  still perfectly reconstruct a function given it's values at the Nyquist frequency. For more details, see Appendix~\ref{app:examples}. 



{\bf Sobolev-Type Balls.\enspace} Sobolev-Type balls are sets of functions that satisfy smoothness constraints on higher-order derivatives~\citep{Barron1993UniversalAB,  Jones1992ASL, Pinkus1985nWidthsIA, Wahba1990SplineMF, Blanchard2022ShallowAD}. One such constraint is that the magnitude of the function's Fourier transform, $\abs{\mathcal{F}f\prn{\omega}}$, decays fast enough as $\abs{\omega}$ approaches infinity. These constraints are imposed by comparing the magnitude of the Fourier transform of the function, $\mathcal{F}f\prn{\omega}$, with the magnitude of $\vert\omega\vert ^r$ for some number $r>0$. In this paper we define a generalized form of Sobolev-type balls:
\begin{equation*}
\mathcal{B}_{r, \rho} ~=~ \crl*{f:\R^d\rightarrow\R, f\in L_2\prn{\R^d}: \frac{1}{\prn{2\pi}^d}\int_{\R^d}\abs{\prn{\mathcal{F}f}\prn{\omega}}~\textnormal{d}\omega, \frac{1}{\prn{2\pi}^d}\int_{\R^d}\abs{\omega}^{r} \abs{\prn{\mathcal{F}f}\prn{\omega}}^{\rho}~\textnormal{d}\omega \leq 1}.
\end{equation*}


In~\citep{Barron1993UniversalAB}, they explored the ability of neural networks to approximate functions in the space $P_1 = \crl{f\in L_2 \prn{\R^d} \mid \int_{\R^d}\abs{\omega}\abs{\prn{\mathcal{F}f}\prn{\omega}} < \infty}$. We now demonstrate that any function in $P_1$ has a normalized representation in $\mathcal{B}_{1,1}$. For this, we show that the conditions on $P_1$ allow us to bound $\norm{\mathcal{F} f}_{L_1 \prn{\R^d}}$. Namely, 
\begin{equation*}
\begin{aligned}
\int_{\R^d} \abs{\prn{\mathcal{F} f} \prn{\omega}}~\textnormal{d}\omega ~&=~  \int_{\brk{-1, 1}^d} \abs{\prn{\mathcal{F} f} \prn{\omega}}~\textnormal{d}\omega + \int_{\R^d {\backslash} \brk{-1, 1}^d} \abs{\prn{\mathcal{F} f} \prn{\omega}}~\textnormal{d}\omega \\ 
~&\leq~  C_1 \prn*{\int_{\brk{-1, 1}^d} \abs{\prn{\mathcal{F} f} \prn{\omega}}^2~\textnormal{d}\omega }^{1/2} + \int_{\R^d \backslash \brk{-1, 1}^d} \abs{\omega}\abs{\prn{\mathcal{F} f} \prn{\omega}}~\textnormal{d}\omega \\ 
~&\leq~  C_1 \norm{f}_2 + \int_{\R^d} \abs{\omega}\abs{\prn{\mathcal{F} f} \prn{\omega}}~\textnormal{d}\omega, 
\end{aligned}
\end{equation*}

for $C_1 > 0$. Therefore, any function in the space $P_1$ can be scaled to a function in $\mathcal{B}_{1, 1}$. Specifically, they showed that sigmoidal neural networks with a bounded depth and $\mathcal{O}(d \epsilon^{-2})$ neurons can approximate such functions with error at most $\epsilon$. 

In this paper we will focus on the space $\mathcal{B}_{2r, 2}$. In Lemma~\ref{lem:ball_in_sob} in the appendix we show that this space is embedded as a proper subspace of $\mathcal{W}^{r, 2}$, with the following norm: 
\begin{equation}\label{def:ball_norm}
\norm{f}_{\mathcal{B}_{2r, 2}} ~=~ \prn*{\normalize \int_{\R^d} (1+\abs{\omega}^{2r}) \abs{\prn{\mathcal{F} f}\prn{\omega}}^{2}~\textnormal{d}\omega}^{1/2}.
\end{equation}
It is worth mentioning that \cite{Pinkus1985nWidthsIA} showed that using traditional basis functions, functions of the space $P_2 = \crl{f\in L_2 \prn{\R^d} \mid \int_{\R^d}\abs{\omega}^{2r}\abs{\prn{\mathcal{F}f}\prn{\omega}}^2 < \infty}$ may be approximated with error at most $\epsilon$ using $\cO(\epsilon^{-d/2r})$ parameters. In this paper, we show that by enforcing an additional constraint on the $L_1$ norm of the Fourier coefficients, we can circumvent exponential dependence on the dimension $d$.

\subsection{Neural Network Architectures}\label{sec:arch}

In the previous section, we described a setting in which a class of candidate functions $\mathcal{H}$ serve as approximators to a class of target functions $\mathcal{U}$. In this work, we compare the approximation guarantees of standard multi-layered perceptrons and a generic set of neural networks that incorporate multiplication operations. Our goal is to understand whether multiplication layers can provide better guarantees to approximate bandlimited functions and members of $\mathcal{B}_{2r, 2}$. 

{\bf Multilayered perceptrons.\enspace} A multilayered perceptron is a type of neural network architecture that consists of $L$ layers of linear transformations interspersed with element-wise non-linear activation functions (e.g., the ReLU function). Typically, the last layer does not include a non-linear activation. 

\begin{definition}[Multilayered perceptron]
A multilayered perceptron is a function $f=y_{L,1}:\R^{p_0}\to \R$ defined by a set of univariate functions $\bigcup^{L}_{i=0}\{y_{i,j}\}^{p_i}_{j=1}$. Each function $y_{i,j}:\R^{p_0}\to\R$ (also known as a neuron) is recursively computed in the following manner 
\begin{equation*}
\begin{aligned}
y_{L,j}(x) ~&=~ \langle w_{L,j},y_{L-1}(x)\rangle + b_{L,j} \\
y_{i,j}(x) ~&=~ \sigma(\langle w_{i,j},y_{i-1}(x)\rangle + b_{i,j})\\ 
y_{0,j}(x) ~&=~ x_j,
\end{aligned}
\end{equation*}
where $i \in [L-1]$, $j \in [p_i]$, and $w_{i,j}\in \R^{p_{i-1}}$ and $b_{i,j}\in\R$ are the weights and a bias of the neuron $y_{i,j}$ and $y_i=(y_1,\dots,y_{p_i})$. The function $\sigma:\R\to\R$ is a non-linear activation function.  
\end{definition}
In this work, we focus on neural networks with ReLU activations, which are defined as $\sigma(x) = \max(0,x)$. However, it is worth noting that other activation functions have been proposed in the literature, such as sigmoidal functions that are measurable functions $\eta$  that satisfy $\eta(x)\rightarrow 0$ as $x\rightarrow -\infty$ and $\eta(x)\rightarrow 1$ as $x\rightarrow \infty$. 

{\bf Multiplicative neural networks.\enspace} In this work, we are interested in comparing the approximation abilities of standard ReLU networks, with that of neural networks that incorporate multiplication layers (also known as product units~\citep{10.1162/neco.1989.1.1.133}). In order to fully understand the added benefits of multiplication gates, we ask the following question: {\em Are multiplications between neurons sufficient to substitute the non-linear activations in multilayered perceptrons?}

\begin{definition}[Multiplicative network]\label{def:second_order} A multiplicative neural network is a function $f=y_{L,1}:\R^{p_0}\to \R$ defined by a set of univariate functions $\bigcup^{L}_{i=1}\{y_{i,j}\}^{p_i}_{j=1}$. Each function $y_{i,j}:\R^{p_0}\to\R$ (also known as a neuron) is recursively computed in the following manner 
\begin{equation*}
\begin{aligned}
y_{i,j}(x) ~&=~ \langle w_{i,j}, y_{i-1}(x)\rangle + a_{i,j} y_{i-1,j_1}(x) y_{i-1,j_2}(x) + b_{i,j}\\ 
y_{0,j}(x) ~&=~ x_j,
\end{aligned}
\end{equation*}
where $a_{i,j}, b_{i,j} \in \R$, $w_{i,j}\in \R^{p_{i-1}}$ are trainable parameters and $y_i=(y_1,\dots,y_{p_i})$.
\end{definition}
A multiplicative network differs from standard multilayered perceptrons in two main ways. First, we do not incorporate non-linearities 
(ReLU activations) between the intermediate layers. In addition, each neuron incorporates a multiplication gate between two preceding neurons.  The total number of trainable parameters in a multiplicative layer is $p_{i+1}p_i+2p_i$, which is $p_i$ more parameters than a fully-connected layer. Following~\cite{Yarotsky2017ErrorBF, Blanchard2022ShallowAD} we measure the complexity of a neural network using its depth $L$ and the number of neurons $G=\sum^{L-1}_{i=1}p_i$ it includes.

{\bf Self-Attention and multiplicative layers.\enspace} Let us describe a single-headed self-attention operation in the original Transformer~\citep{NIPS2017_3f5ee243}. Each layer $i \in [L]$ of a depth-$L$ Transformer encoder is defined as follows. The input to the $i$th layer is a sequence of $N$ tokens, denoted by $x_i=\crl{x_{i, j}}_{j=1}^{N}$, where each $x_{i,j} \in \R^{d_x}$ represents the $j$th token of the $i$th layer. To compute the output of the $i$th layer at a particular position $e\in\brk{N}$, we use the following formula:
\[
f_\textnormal{SA}^{i, e} (x_i) ~=~ \sum_{j=1}^N \textnormal{softmax}_j \left(\fr{1}{\sqrt{d_a}} \langle W^{Q, i} x_{i, e}, W^{K, i} x_{i, e}\rangle \right) W^{V, i} x_{i, j},
\]
where $\textnormal{softmax}_j(f(x)) = {\exp(f(x)_j)} / { \sum_{j'} \exp(f(x)_{j'})}$ is the softmax operator, and the trainable weight matrices $W^{K,i}$, $W^{Q,i}$, $W^{V,i} \in \R^{d_a \times d_x}$ convert the representation from its dimension $d_x$ into the attention dimension $d_a = d_x$, creating `Key', `Query', and `Value' representations, resp. As can be seen, the self-attention layers use multiplicative connections when computing the following inner product $\langle W^{Q, i} x_{i, e}, W^{K, i} x_{i, e} \rangle$. This operation computes multiplications between the coordinates of transformations of the same token $x_{i,e}$. In other words, it can be thought of as computing a multiplicative layer on an input vector $x_{i,e}$. As a side note, in addition to self-attention layers, transformers also incorporate commonly used layers such as fully-connected layers, residual connections, and normalization layers, which are not the focus of this paper.



\section{Representation Power of Multiplicative Neural Networks}

In this section, we explore the expressive power of neural networks with multiplication layers. We first demonstrate that these networks can easily represent polynomial functions, which we then use to approximate bandlimited functions. This allows us to approximate functions in the space $\mathcal{B}_{2r, 2}$ without suffering from the curse of dimensionality. Specifically, we prove the following lemma:

\begin{restatable}{lemma}{polyRealize}\label{lemma:poly_realization}
For any polynomial $p_n:\R \rightarrow \R$ of degree $n$ of the form $p_n(x) = \sum_{k=0}^{n} c_k x^k$, there exists a multiplicative neural network $\nn{POL}_{n}:\R^3\to\R$, of depth $L_n = \mathcal{O}(n)$ with $G_n=\mathcal{O}(n)$ neurons that satisfies $\nn{POL}_{n}\prn{x, x, c_0} = p_n\prn{x}$ for $c_0\in\R$, and $x\in\R$.
\end{restatable}

With this lemma in hand, we can show how to use multiplicative networks to approximate analytic kernel functions $K:\R\rightarrow \C$ by leveraging their ability to represent polynomials. By approximating a certain class of polynomials, we can demonstrate how to use these networks to approximate analytic kernels. Once we have the ability to approximate analytic kernels, we can use Maurey's Theorem (see Lemma~\ref{lem:maurey}) to express any function $f\in\cH_{M}$ as a finite sum of kernel superpositions, and construct a network that approximates this sum.

{\bf Approximating of real-valued analytic functions.\enspace} In this section, we show how to approximate certain real-valued analytic functions. For this purpose, we recall the notion of the Bernstein s-ellipse, which is a geometric shape defined on the complex plane that is useful in approximation theory.
\begin{definition}
Let $M\geq1$, $s>1$ be two scalars. The Bernstein s-ellipse on $[-M, M]$ is defined as follows
\[
E_s^M ~=~ \crl*{x+i y \in \C : \fr{x^2}{\prn{a_s^M}^2}+\fr{y^2}{\prn{b_s^M}^2}=1},
\]
whose semi-major and semi-minor axes are $a_s^M = M\fr{s+s^{-1}}{2}$ and $b_s^M = M\fr{s-s^{-1}}{2}$.
\end{definition}
The parameter $s$ controls the shape of the ellipse, and the parameter $M$ determines its size. For example, when $s=2$, the ellipse is a circle centered at the origin with a radius $M$. As $s$ increases, the ellipse becomes more elongated and its semi-minor axis decreases. The semi-major and semi-minor axes of the ellipse determine its maximal and minimal values, respectively. Before stating our result in Theorem~\ref{thm:deep_som_for_analytic}, we recall the following theorem of \cite{Trefethen2019ApproximationTA}:

\begin{theorem}[Theorem 8.2 of \cite{Trefethen2019ApproximationTA}]\label{thm:chebychev_unique}
Let $M >0, s>2$ be scalars and $K:[-M,M]\to\R$ be an analytic function that is analytically continuable to the ellipse $E_s^M$, where it satisfies $\sup_{x \in E_s^M}\abs{K\prn{x}}\leq C_K$ for some constant $C_K>0$. For every $n\in \mathbb{N}$, there exists a polynomial $h_n: \R\rightarrow\C$ of degree $n$, such that, 
\[
\norm{h_n -K}_{L^{\infty} \prn{\brk{-M, M}}} ~\leq~ \frac{2 C_K s^{-n}}{s-1}.
\]
\end{theorem}

Theorem~\ref{thm:chebychev_unique} states that any analytic function $K$ that is bounded on the Bernstein s-ellipse $E_s^M$ can be accurately approximated by a polynomial $h_n$ of degree $n$, with the error decreasing exponentially as the degree $n$ increases. 

As we show next, the use of multiplication layers in neural networks can improve the efficiency of function approximation in certain cases. The following theorem shows that deep multiplicative networks can approximate real-valued analytic functions on bounded intervals.

\begin{restatable}{theorem}{analyticApprx}
\label{thm:deep_som_for_analytic} Let $M\geq1$, $s>2$, $C_K>0$ and $\epsilon\in (0,1)$ be scalars. Then, for any real-valued analytic function $K:\brk{-M, M}\to\R$ that is analytically continuable to the ellipse $E_s^M$ where $\abs{K\prn{x}}\leq C_K$, there exists a deep multiplicative network $\nn{MA}:[-M,M]^3 \to \R$ (MA stands for `Multiplicative Analytic') of depth $L_{\epsilon}=\cO\prn{\fr{1}{\log_2 s} \log_2 \fr{C_K}{\epsilon}}$ with $G_{\epsilon}=\cO\prn{\fr{1}{\log_2 s} \log_2 \fr{C_K}{\epsilon}}$ neurons that satisfies
\[
\norm{\nn{MA}\prn{x, x, x} -K\prn{x}}_{L^{\infty} \prn{\brk{-M, M}}} ~\leq~ \epsilon.
\]
\end{restatable}

Theorem~\ref{thm:deep_som_for_analytic} establishes that deep multiplicative networks with second-degree multiplications can approximate any real-valued analytic function that is bounded on the Bernstein s-ellipse $E_s^M$ with error bounded by a quantity that decreases exponentially with the depth of the network.

In~\citep{Montanelli2021DeepRN}, the authors show that the kernel may be approximated with depth $L_{\epsilon}=\cO\prn{\fr{1}{\log^2_2 s} \log^2_2 \fr{C_K}{\epsilon}}$ with $G_{\epsilon}=\cO\prn{\frac{1}{\epsilon^2 \log^2_2 s} \log^2_2 \frac{C_K}{\epsilon}}$ neurons.  In contrast, our approach achieves a $\cO\prn{\log(1/\epsilon)}$ improvement in depth and a $\cO\prn{\log(1/\epsilon^2)}$ in number of neurons, resulting in an exponentially faster convergence to the target function. 



{\bf Approximation of bandlimited functions.\enspace} As a next step, we study the ability of neural networks to approximate bandlimited functions. We start by showing how bandlimited functions can be approximated using analytic kernels. For this purpose, we will use Maurey's theorem~\citep{SAF_1980-1981____A5_0,Milman1986} that states that for Hilbert subspaces with a bounded norm, any function in the convex hull can be easily approximated using points from the subspace, where the rate of approximation is dependent on the number of points used.  

The following theorem shows that a deep multiplicative network can approximate in $B=\brk{0,1}^d$, a bandlimited function up to a given error tolerance using a relatively small number of neurons and depth. 

\begin{restatable}{theorem}{bandlimitedApprx}\label{thm:bandlimitedApprx}
Let $\epsilon \in (0,1)$, $M>1, d\geq2$, and $K:\R\rightarrow\C$ be an analytic kernel that holds the assumptions of Theorem \ref{thm:deep_som_for_analytic} with respect to $s>2, C_K>0$, and bounded by a constant $D_K\in(0, 1]$ on $\brk{-dM, dM}$. Let $f$ be a real-valued function in $\cH_{K, M}\prn{B}$. Further, let  $F:\brk{-M, M}^d \rightarrow \C$ be a square-integrable function such that $f(x) = \smallint_{\brk{-M, M}^d} F\prn{\omega}K\prn{\omega \cdot x} ~\textnormal{d}\omega$. We define $C_F = \int_{\R^d} \abs{F\prn{\omega}} ~\textnormal{d}\omega$ = $\int_{\brk{-M, M}^d} \abs{F\prn{\omega}} ~\textnormal{d}\omega$. Then, there exists a deep multiplicative network $\nn{MBL}:B \to \R$ (MBL stands for `Multiplicative bandlimited') of depth $L_{\epsilon}=\cO\prn*{\fr{1}{\log_2 s} \log_2 \fr{C_F C_K }{\epsilon}}$ with $G_{\epsilon}=\cO\prn*{\fr{C^2_F}{\epsilon^2  \log_2 s}\log_2{\fr{C_F C_K}{\epsilon}}}$ neurons that satisfies $\norm{\nn{MBL}-f}_{L^2 \prn{B}} \leq \epsilon$.
\end{restatable}

The above theorem shows that one can approximate bandlimited functions up to error $\epsilon$ using multiplicative neural networks of depth $L_{\epsilon}=\cO\prn*{\log(1/\epsilon)}$ using $G_{\epsilon}=\cO\prn*{\epsilon^{-2}\log(1/\epsilon)}$ neurons. In comparison,~\cite{Montanelli2021DeepRN} showed that one can approximate bandlimited functions to the same level of approximation using standard ReLU networks of depths $L_{\epsilon}=\cO\prn*{\log^2(1/\epsilon)}$ with $G_{\epsilon}=\cO\prn*{\epsilon^{-2}\log^2(1/\epsilon)}$ neurons. This result demonstrates the inherent parameter efficiency of multiplicative neural networks in comparison with standard ReLU networks.




{\bf Approximation of Smooth Functions.\enspace} We now turn to show results for Sobolev-Type functions. We use the results on bandlimited functions shown in Theorem \ref{thm:bandlimitedApprx} to approximate functions in $\mathcal{B}_{2r, 2}\subsetneq \mathcal{W}^{r, 2}$. We show that using slightly stronger assumptions, we get an approximation rate comparable to those shown by \cite{Barron1993UniversalAB} (where the network is a shallow sigmodial network model) using multiplicative neural networks (i.e. without non-linear activations). Further, since these are in fact in $\mathcal{W}^{r, 2}$, we may better characterize such functions.


\begin{restatable}{theorem}{strictSobRS}
\label{thm:strict_sobolev_rs} Let $d \geq 2,r \in \mathbb{N}$, $f\in\mathcal{B}_{2r, 2}$ and $\epsilon>0$. There exists a deep ReLU network $\nn{RS}$ (standing for ``ReLU Sobolev'') with a depth of $L_{\epsilon} = \cO(d^2 \epsilon^{-2/r})$ and $G_{\epsilon}= \cO(d^2 \epsilon^{-\prn{2+2/r}})$ neurons, such that $\norm*{\nn{RS} - f}_{L_2 \prn{B}} \leq \epsilon$. 
\end{restatable}

\begin{restatable}{theorem}{strictSobMS}\label{thm:strict_sobolev_ms} Let $d \geq 2,r \in \mathbb{N}$, $f\in\mathcal{B}_{2r, 2}$ and $\epsilon>0$. There exists a deep ReLU network $\nn{MS}$ (standing for ``Multiplicative Sobolev'') with a depth of $L_{\epsilon}= \cO(d \epsilon^{-1/r})$ and $G_{\epsilon}= \cO(d \epsilon^{-\prn{2 + \prn{1/r}}})$ neurons, such that $\norm*{\nn{MS} - f}_{L_2 \prn{B}} \leq \epsilon$. 
\end{restatable}

\begin{proof} Let $f\in\mathcal{B}_{2r, 2}$. We would like to approximate $f$ using a bandlimited function $f_M$ and then approximate $f_M$ using a multiplicative neural network  $\nn{BL}=\nn{MS}$. Let $M >1$ be a band. 

We recall the Inverse Fourier transform given by $\prn{\mathcal{F}^{-1} g}\prn{x} = \frac{1}{\prn{2\pi}^d} \int_{\R^d} g\prn{\omega} \exp\prn{i\omega \cdot x} ~\textnormal{d}\omega$. We define the bandlimiting of $f:\R^d\rightarrow\R$ as $f_M = \mathcal{F}^{-1} \prn{\mathcal{F} {f}\mathbbm{1}_{\brk{-M, M}^d}}$, such that $f_M \in {\cH_{K, M}\prn{B}}$ for {$K\prn{x} = \exp(ix)$ and $F= \normalize \mathcal{F}f$}. We have
\begin{equation*}
\begin{aligned}
\norm{f-f_M}_{L_2 \prn{B}} ~&\leq~ 
\prn*{\frac{1}{\prn{2\pi}^d}\norm*{\cF f-\cF f\mathbbm{1}_{\brk{-M, M}^d}}_{L_2 \prn{\R^d}}^2}^{1/2} \\ 
~&=~ \prn*{\normalize\int_{\R^d \backslash \brk{-M, M}^d}{\abs*{\cF f\prn{\omega}}^2}~\textnormal{d}\omega}^{1/2}.
\end{aligned}
\end{equation*}

For any $\omega\in \R^d \setminus \brk{-M, M}^d$, we have $\abs{M^{-1} \omega}^{2r} \geq 1$. Therefore, 
\begin{equation*}
\begin{aligned}
\prn*{\normalize\int_{\R^d \backslash \brk{-M, M}^d}{\abs*{\cF f\prn{\omega}}^2}~\textnormal{d}\omega}^{1/2} ~&\leq~ \prn*{\normalize\int_{\R^d \backslash \brk{-M, M}^d}\abs*{M^{-1}\omega}^{2 r}{\abs*{\cF f\prn{\omega}}^2}~\textnormal{d}\omega}^{1/2}\\ 
~&\leq~ M^{-r} \prn*{\normalize\int_{\R^d}{\abs{\omega}^{2r} \abs*{\cF f\prn{\omega}}^2}~\textnormal{d}\omega}^{1/2} \\ 
~&\leq~ M^{-r},
\end{aligned}
\end{equation*}
where the final inequality is due to $f\in\mathcal{B}_{2r, 2}$. For any $\epsilon>0$, we set $M = \prn{2/\epsilon}^{1/r}$. We will construct a neural network $\nn{MS}$ to approximate the bandlimited function $f_M$, such that
\begin{equation*}
\norm{f_M - \nn{MS}}_{L_2 \prn{B}} ~\leq~ \epsilon/2.
\end{equation*}
Assuming we have constructed such $\nn{MS}$, then by the triangle inequality we then arrive at
\begin{equation*}
\begin{aligned}
\norm{f - \nn{MS}}_{L_2 \prn{B}} 
~\leq~ \norm{f - f_M}_{L_2 \prn{B}} + \norm{f_M - \nn{MS}}_{L_2 \prn{B}} ~\leq~ M^{-r} + \epsilon/2 ~\leq~ \epsilon.
\end{aligned}
\end{equation*}

We define a function $F:\brk{-M, M}^d \rightarrow \C$ as follows
\[
F\prn{\omega} ~=~ \normalize \prn{\mathcal{F} f}\prn{\omega}.
\]
We then have the following identity
{\[
f_M (x) ~=~ \int_{\brk{-M, M}^d} F \prn{\omega}K\prn{\omega \cdot x} ~\textnormal{d}\omega ~=~ \int_{\brk{-M, M}^d} \fr{1}{(2\pi)^d} \mathcal{F}f \prn{\omega}  \exp\prn{i\omega\cdot x} ~\textnormal{d}\omega.
\]}
We may now work under the conditions of Theorem~\ref{thm:bandlimitedApprx}. We consider that {$C_{F} = \int_{\brk{-M, M}^d} \abs{F\prn{\omega}}~\textnormal{d}\omega = \normalize\int_{\brk{-M, M}^d} \abs{\cF{f}\prn{\omega}}~\textnormal{d}\omega \leq 1$}, where the inequality is due to the definition of $\mathcal{B}_{2r, 2}$. The kernel {$K\prn{t} = \exp\prn{i t}$} takes as input $t=\omega \cdot x$, for $x \in B$ and $\omega\in\brk{-M, M}^d$. Therefore, $t\in\brk{-d M, d M}$, and $K:\brk{-d M, d M}\rightarrow\R$. We note that $K$ is continuable to the Bernstein 4-ellipse $E_4^{d M}$. We notice that $a^{d M}_4 = d M(4+4^{-1})/2$ is the larger axis, and therefore the maximal norm of $K$ on $E^{d M}_s$ is given by $K(a^{dM}_4)$:
{\[
C_K ~=~ \max_{t} \abs{K\prn{t}} ~\leq~ \exp\prn{dM \fr{4+4^{-1}}{2}}.
\]}
\noindent{Further, for any $t\in\R$ we have $\abs{K(t)}\leq 1 = D_K$.}

We now approximate $f_M$ with a multiplicative network. Using the results of Theorem \ref{thm:bandlimitedApprx}, there exists a deep multiplicative neural network $\nn{MS}$ that approximates the bandlimited function $f_M$ in $L_2 \prn{B}$ with error bounded by $\epsilon/2$ and depth
\begin{equation}
\begin{aligned}
L_{\epsilon} ~&\leq~ C_1 \frac{1}{\log_2 4} \log_2 {\frac{2 C_K C_F }{\epsilon}} \\ ~&\leq~ C_2 \frac{1}{\log_2 4} {\log_2 {\frac{{\exp\prn{d M \frac{4 + 4^{-1}}{2}}}}{\epsilon}}} \\ ~&\leq~ C_2 \prn*{3d \epsilon^{-1/r}  + \log_2 (1/\epsilon)} \\ ~&\leq~ C_3\cdot d \epsilon^{-1/r},
\end{aligned}
\end{equation}
for some constants $C_1,C_2,C_3>0$. In addition, the number of neurons can be bounded by
\begin{equation}
\begin{aligned}
G_{\epsilon} ~&\leq~ C'_1 {C_F^2} \cdot \epsilon^{-2} \log_2 {\frac{2 C_K C_F }{\epsilon}} \\ ~&\leq~ C'_2 \cdot \epsilon^{-2} {\log_2 {\frac{{\exp\prn{d M \frac{4 + 4^{-1}}{2}}}}{\epsilon}}} \\ 
~&\leq~ C'_2 \cdot \epsilon^{-2} \prn{3d (2/\epsilon)^{1/r} + \log_2(1/\epsilon)} \\ 
~&\leq~ C'_3\cdot d \epsilon^{-\prn{2 + 1/r}},
\end{aligned}
\end{equation}
for some constants $C'_1,C'_2, C'_3>0$. 
\end{proof}

Theorems~\ref{thm:strict_sobolev_rs}-\ref{thm:strict_sobolev_ms} provide insights into several interesting properties of the Sobolev-Type ball $\mathcal{B}_{2r, 2}$. First, we see that approximating these functions does not suffer from the curse of dimensionality that occurs when approximating the full Sobolev space. Secondly, for both ReLU networks and multiplicative networks, the bound becomes tighter as the smoothness $r$ increases, which is a desirable property for approximation error bounds. In fact, as $r$ approaches infinity, the bound approaches the one proposed in \citep{Barron1993UniversalAB}. Lastly, we show that for the same error tolerance $\epsilon$, multiplicative neural networks can approximate a target function $f \in\mathcal{B}_{2r, 2}$ with a depth of $\cO(d^2\epsilon^{-1/r})$ and $\cO(d^2\epsilon^{-(1+1/r)})$ neurons, while standard ReLU neural networks require a depth of $\cO(d^2\epsilon^{-2/r})$ and $\cO(d^2\epsilon^{-(2+2/r)})$ neurons. This result demonstrates that multiplicative neural networks have stronger approximation guarantees when approximating functions in the Sobolev-Type space.

\section{Conclusions}\label{sec:conclusions}

Previous papers have studied the approximation guarantees of standard fully-connected neural networks to approximate functions in the Barron space $\mathcal{B}_{1,1}$~\citep{Barron1993UniversalAB}, the space of bandlimited functions~\citep{Montanelli2021DeepRN}, and the Korobov space~\citep{Blanchard2022ShallowAD}. These studies have shown that fully-connected networks can approximate a wide range of smooth functions without suffering from the curse of dimensionality and have provided insights into the tradeoffs between the width and depth of neural networks in learning certain types of functions. However, these results are limited to variants of fully-connected network and do not provide information about other types of architectures. 

In this paper, we extend these results by exploring the approximation guarantees of both multiplicative neural networks and standard fully-connected networks to approximate bandlimited functions and members of the Sobolev-Type ball $\mathcal{B}_{2r, 2}$. Our results show that multiplicative neural networks achieve stronger approximation guarantees compared to standard ReLU networks. In addition, we show that, unlike the Barron space and the space of bandlimited functions, $\mathcal{B}_{2r, 2}$ is a subset of the Sobolev space $\sob{2}$. Therefore, our results demonstrate that it is possible to avoid the curse of dimensionality for wide subsets of the Sobolev space.

\section*{Acknowledgements}

Tomer Galanti was supported by the Center for Minds, Brains and Machines (CBMM), funded by
NSF STC award CCF-1231216 and by NSF award 213418.

\newpage

\bibliography{tmlr_bib}
\bibliographystyle{tmlr}

\newpage

\appendix

\section{Examples}\label{app:examples}

\begin{example}\label{ex:BL}
Let $f\in{L}_{2} \prn{\R ^d}$ such that $f$ is Bandlimited function with band $\pi$ (i.e., $\supp \prn{\mathcal{F}f} \subset \brk{-\pi, \pi}^d$). Let $\phi\prn{x} = \sinc\prn{x}= \sin\prn{\pi x} / {\pi 
 x}$, and $\phi_k= \phi\prn{{\cdot} -k}$ for all $k\in\Z^{d}$. By the Shannon-Nyquist theorem~\citep{Shannon1984CommunicationIT}, we have: 
\begin{equation*}
\begin{aligned}
f\prn{x}~=~\sum_{k\in\Z^d} \tri{f, \phi_k} \phi_k \prn{x}~=~\sum_{k\in\Z^d} f\prn{k} \phi{\prn{x-k}}.
\end{aligned}
\end{equation*}
This means that every Bandlimited function can be completely determined using a discrete set of integer samples. This result is particularly surprising for high-dimensional functions ($d$ is large) since  the maximal distance between a point $x$ and a sampling point grows with $d$. For example, the vertex of the unit cube in $\R^d$ is of distance $\sqrt{d}/2$ away from its center. Despite this, we can still recover samples from the integer vertices, even when the distances scale as $\sqrt{d}$. This property is useful when recovering high-dimensional functions using neural networks. See Figure~\ref{fig:example} for illustration.
\end{example}

\begin{figure}[h]
\begin{center}
\begin{tikzpicture}[x=3cm, y=3cm, >=stealth]

\draw (0,0.5) -- (1,0.5);
\draw (0,0.5) -- (0.5,0.5) node [above, midway] {\small $\frac{1}{2}$};
\node[draw,circle,inner sep=1pt, fill=mycolor] at (0,0.5) {};
\node[draw,circle,inner sep=1pt, fill=mycolor] at (1,0.5) {};

\node[draw,circle,inner sep=1pt, fill=mycolor_orange] at (0.5,0.5) {};

\draw (2,0) -- (3,0) -- (3,1) -- (2,1) -- cycle;
\draw (2,0) -- (2.5,0.5) node [above, midway] {\small $\frac{\sqrt{2}}{2}$};

\node[draw,circle,inner sep=1pt, fill=mycolor_orange] at (2.5,0.5) {};
\node[draw,circle,inner sep=1pt, fill=mycolor] at (2,0) {};
\node[draw,circle,inner sep=1pt, fill=mycolor] at (3,0) {};
\node[draw,circle,inner sep=1pt, fill=mycolor] at (3,1) {};
\node[draw,circle,inner sep=1pt, fill=mycolor] at (2,1) {};

\draw (4,0,0) -- (5,0,0) -- (5,1,0) -- (4,1,0) -- cycle;
\draw (4,0,1) -- (5,0,1) -- (5,1,1) -- (4,1,1) -- cycle;
\draw (4,0,0) -- (4,0,1);
\draw (5,0,0) -- (5,0,1);
\draw (5,1,0) -- (5,1,1);
\draw (4,1,0) -- (4,1,1);

\draw (5,0,1) -- (4.5,0.5,0.5) node [above, midway] {\small $\frac{\sqrt{3}}{2}$};
\node[draw,circle,inner sep=1pt, fill=mycolor_orange] at (4.5,0.5,0.5) {};

\node[draw,circle,inner sep=1pt, fill=mycolor] at (4,0,0) {};
\node[draw,circle,inner sep=1pt, fill=mycolor] at (5,0,0) {};
\node[draw,circle,inner sep=1pt, fill=mycolor] at (5,1,0) {};
\node[draw,circle,inner sep=1pt, fill=mycolor] at (4,1,0) {};

\node[draw,circle,inner sep=1pt, fill=mycolor] at (4,0,1) {};
\node[draw,circle,inner sep=1pt, fill=mycolor] at (5,0,1) {};
\node[draw,circle,inner sep=1pt, fill=mycolor] at (5,1,1) {};
\node[draw,circle,inner sep=1pt, fill=mycolor] at (4,1,1) {};

\end{tikzpicture}

\end{center}
\caption{Illustration of Example \ref{ex:BL}. A bandlimited function $f$ can be reconstructed using a discrete set of values, despite the fact that the distances in each cube grow as $\mathcal{O}(\sqrt{d})$.}
\label{fig:example}
\end{figure}
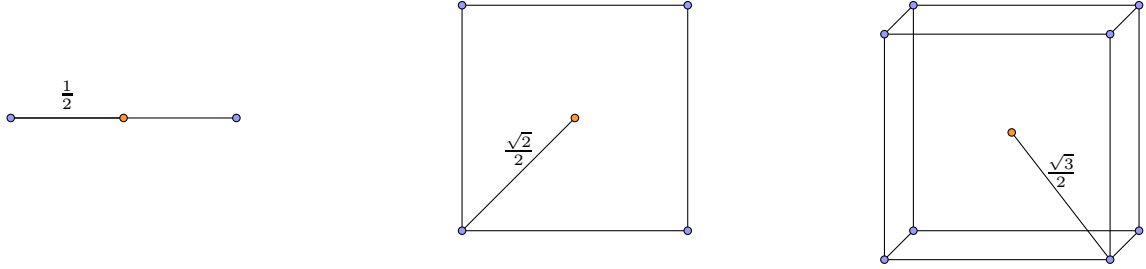

\section{Proofs}\label{app:proofs}

\begin{definition}
    For two Banach spaces $V_1$ and $V_2$, we say that $V_1$ is embedded in $V_2$ if $\norm{u}_{V_2} \leq C\norm{u}_{V_1}$ for some constant $C$ and for $u \in V_1$.
\end{definition}

\begin{restatable}{lemma}{ballSob}\label{lem:ball_in_sob}
Let $r>0$ and $d \geq 2$. Then, $\mathcal{B}_{2r, 2}$ is embedded as a proper subspace of  $\mathcal{W}^{r, 2}$. 
\end{restatable}

\begin{proof}
Since $f\in\mathcal{B}_{2r, 2}$, we have $f\in L_2 \prn{\R^d}$, and therefore, $\norm{f}_2<\infty$. Additionally, by \cite{book:6771}(Theorem 5.2), there exists a constant $C_1>0$ such that 
\[
\norm{f}_{\mathcal{W}^{r, 2}} ~\leq~ C_1 \prn{\norm{f}_2 + \abs{f}_{r, 2}}, 
\]
where $C_1$ is dependant on $r$ and the dimension $d$. Therefore, it remains to prove that $\abs{f}_{r, 2}$ is bounded. For this, we use some properties of multivariate function spaces. 

We begin with the following claim that we will use in advance. Let $\omega\in\R^d, d\geq 2$, we have
\begin{equation}\label{eq:w2r}
\abs{\omega}^{2r} ~=~ \norm{\omega}_{l_2}^{2r} ~=~ \prn*{\sum_{m=1}^d \omega_m^2}^{r} ~=~ \sum_{\|\alpha\|_1=r} \binom{r}{\alpha_1,\dots,\alpha_d} \cdot \prn{\omega^{\alpha}}^2, 
\end{equation}
where $\omega^{\alpha}=\prod^{d}_{i=1}\omega^{\alpha_i}$ for $\alpha\in\R^d$.

By Jensen's inequality, 
\begin{equation*}
\begin{aligned}
\abs{f}_{r, 2}^2 ~=~ \prn*{\sum_{\alpha:\abs{\alpha}=r} \norm{D^{\alpha} f}_2}^2 ~\leq~  d^{r}\sum_{\alpha:\abs{\alpha}=r} \norm{D^{\alpha} f}_2^2
\end{aligned}
\end{equation*}
By Parseval's Identity in $\R^d$~\citep{Albrecht2021OperatorTA}:
\[
\norm{D^{\alpha} f}_{2}^2 ~=~ \normalize \int_{\R^d} \abs{ \cF\prn{ D^{\alpha} f}\prn{\omega}}^{2}~\textnormal{d}\omega ~=~ \normalize \int_{\R^d} \abs{\prn{\cF f}\prn{\omega}}^2 {\abs{\omega^{\alpha}}^2}~\textnormal{d}\omega.
\]
Hence, by \eqref{eq:w2r}
\begin{equation*}
\begin{aligned}
\sum_{\alpha:\abs{\alpha}=r} \norm{D^{\alpha} f}_2^2 
~&=~  \sum_{\alpha:\abs{\alpha}=r} \normalize \int_{\R^d} \abs{\prn{\cF f}\prn{\omega}}^2 {\abs{\omega^{\alpha}}^2} ~\textnormal{d}\omega \\ 
~&\leq~ \normalize \int_{\R^d} \abs{\prn{\cF f}\prn{\omega}}^2 \prn*{\sum_{\alpha:\abs{\alpha}=r} \binom{r}{\alpha_1,\dots,\alpha_d} \cdot {\abs{\omega^{\alpha}}^2} }~\textnormal{d}\omega \\ 
~&=~ \normalize \int_{\R^d} \abs{\cF f\prn{\omega}}^2 \abs{\omega}^{2r}~\textnormal{d}\omega.
\end{aligned}
\end{equation*}
where the inequality follows from the fact that $\binom{r}{\alpha_1,\dots,\alpha_d} \geq 1$. Finally, we conclude that
\begin{equation*}
\begin{aligned}
\norm{f}_{\sob{2}} ~\leq~ C_1 \prn{\norm{f}_2 + \abs{f}_{r, 2}}
~&\leq~ C_1 \prn*{\norm{f}_2 + \fr{d^r}{(2\pi)^{d/2}}\left(\smallint_{\R^d} \abs{\prn{\cF f}\prn{\omega}}^2 \abs{\omega}^{2r}~\textnormal{d}\omega\right)^{1/2}}  \\~&\leq~ C_2 \norm{f}_{\mathcal{B}_{2r, 2}}, 
\end{aligned}
\end{equation*}
for some $C_1, C_2 >0$, where the second inequality is due to the definition of $\norm{\cdot}_{\mathcal{B}_{2r, 2}}$ as given in equation \ref{def:ball_norm}.

To see that $\mathcal{B}_{2r, 2}$ is a proper subspace of $\mathcal{W}^{r, 2}$, let us define as an example, the function $f\in\sob{2}\prn*{\R^d}$ through its Fourier transform:

\[
\mathcal{F} f \prn{\omega} = \begin{cases}
  1 & \abs{\omega} \leq 1 \\
  \abs{\omega}^{-\prn{r + (d+\epsilon)/2}} & \abs{\omega}>1 
\end{cases}
\]
for any $\epsilon>0$. 

Indeed, $f\in\sob{2}\prn*{\R^d}$ since:
\begin{equation*}
\begin{aligned}
\int_{\R^d} \abs{\omega}^{2r} \abs{\mathcal{F} f \prn{\omega}}^2 ~\textnormal{d}\omega ~&=~ \int_{\abs{\omega}\leq 1} \abs{\omega}^{2r} \abs{\mathcal{F} f \prn{\omega}}^2 d\omega + \int_{\abs{\omega} > 1} \abs{\omega}^{2r} \abs{\mathcal{F} f \prn{\omega}}^2~\textnormal{d}\omega \\~&\leq~ 1+ \int_{\abs{\omega} > 1} \abs{\omega}^{-\prn{d+\epsilon}}~\textnormal{d}\omega < \infty.
\end{aligned}
\end{equation*}

However, for this example $\mathcal{F} f \prn{\omega} \notin L_1$ in the cases where the dimension is relatively higher than the smoothness index. That is, whenever
\[
r+\prn{d+\epsilon}/2 \leq d \Leftrightarrow 2r + \epsilon \leq d.
\]

This implies that $f$ does not satisfy the condition $\normalize\int_{\R^d} \abs{\mathcal{F} f\prn{\omega}}d\omega \leq 1$ which is required in $\mathcal{B}_{2r, 2}$. We see that the $L_1$ condition on the Fourier transform of functions in $\mathcal{B}_{2r, r}$ allows us to circumvent the curse of dimensionality when approximating in Sobolev spaces. 
\end{proof}

\polyRealize*

\begin{proof}

Let $p_n \prn{x} = \sum_{k=0}^n c_k x^k$ be a polynomial of degree $n$ and $p_{n, i} = \sum_{k=0}^{i} c_k x^k$ its partial sum up to term $i$. We construct a network $\nn{POL}_{n}$ whose $i$th layer satisfies:
\[
\prn{\nn{POL}_{n}}_{i}\prn{x, x, c_0} = \prn{x, x^{i+1}, p_{n,i}(x)}.
\]
At layer $i$, the model weights for the three neurons are defined as:
\[
w_{i, 1} = (1, 0, 0)
,\quad 
a_{i, 2} = 1 \text{   for  } (j_1, j_2) = (1, 2) \text { }
,\quad
w_{i, 3} = (0, c_i, 1)
\]
The rest of the weights are zeros. 

We argue that at any layer $i\geq0$, neuron 1 contains $x$, neuron 2 contains $x^{i+1}$ and neuron 3 contains $p_{n,i}(x)$. Let $y_{{i+1}, 1}, y_{{i+1}, 2}, y_{{i+1}, 3}$ be the three neurons in layer $i+1$. Since the only non-zero weight affecting the first neuron is in $w_{i+1, 1}$, $y_{i+1, 1}=y_{i, 1}=x$ by the assumption. The only non-zero weight affecting the second neuron appears in $a_{i=1, 2}$, and therefore $y_{i+1, 2}=y_{i, 1} y_{i, 2}=x^{i+1}$. Lastly, $y_{i+1, 3}$ depends only on $w_{i+1, 3}$, and so $y_{i+1, 3} = c_{i+1} y_{i, 2} + y_{i, 3}=p_{n,i} \prn{x} + c_{i+1} x^{i+1}: = p_{n,i+1} \prn{x}$. We conclude using the fact that $p_{n,n} \prn{x} = p_n \prn{x}$. 
\end{proof}

\begin{figure}
\begin{center}
\begin{tikzpicture}[x=1.5cm, y=1.5cm, >=stealth]


\node [every neuron/.try, neuron 1/.try, fill=mycolor_orange] (input-1) at (0,2.5-1) { $x$};

\node [every neuron/.try, neuron 2/.try, fill=mycolor_orange] (input-2) at (0,2.5-2) { $x$};

\node [every neuron/.try, neuron 3/.try, fill=mycolor_orange] (input-3) at (0,2.5-3) { \tiny $c_0$};


\node [every neuron/.try, neuron 1/.try] (hidden-1) at (2,2.5-1) { $x$};

\node [every neuron/.try, neuron 2/.try] (hidden-2) at (2,2.5-2) { $x^2$};

\node [every neuron/.try, neuron 3/.try] (hidden-3) at (2,2.5-3) { \tiny $c_0+c_1 x$};


\node [every neuron/.try, neuron 1/.try] (output-1) at (4,2.5-1) { $x$};

\node [every neuron/.try, neuron 2/.try] (output-2) at (4,2.5-2) { $x^{3}$};

\node [every neuron/.try, neuron 3/.try, fill=mycolor] (output-3) at (4,2.5-3) {\small $p_{2, 2} \prn{x}$};


\foreach \l [count=\i] in {1,2,3}
  \draw [<-] (input-\i) -- ++(-1,0)
    node [above, midway] {};




\draw [->] (output-3) -- ++(1,0)
    node [above, midway] {};

\draw [->] (input-1) -- (hidden-1) node [above, midway] {\tiny $(w_{i, 1})_1 = 1$};
\draw [->] (input-1) -- (hidden-2) node [above, midway] {\tiny $a_{i, 2} = 1$};
\draw [->] (input-2) -- (hidden-2) node [above, midway] {\tiny $a_{i, 2} = 1$};
\draw [->] (input-2) -- (hidden-3) node [above, midway] {\tiny $(w_{i, 3})_2 = c_i$};
\draw [->] (input-3) -- (hidden-3) node [above, midway] {\tiny $(w_{i, 3})_3 = 1$};

\draw [->] (hidden-1) -- (output-1) node [above, midway] {\tiny $(w_{i, 1})_1 = 1$};
\draw [->] (hidden-1) -- (output-2) node [above, midway] {\tiny $a_{i, 2} = 1$};
\draw [->] (hidden-2) -- (output-2) node [above, midway] {\tiny $a_{i, 2} = 1$};
\draw [->] (hidden-2) -- (output-3) node [above, midway] {\tiny $(w_{i, 3})_2 = c_i$};
\draw [->] (hidden-3) -- (output-3) node [above, midway] {\tiny $(w_{i, 3})_3 = 1$};

\foreach \l [count=\x from 0] in {Input, Hidden, Output}
  \node [align=center, above] at (\x*2,2) {\l \\ layer};

\end{tikzpicture}
\end{center}
\caption{Illustration of the multiplicative network in Proof of Lemma \ref{lemma:poly_realization}, where the polynomial degree $n=2$.}
\end{figure}
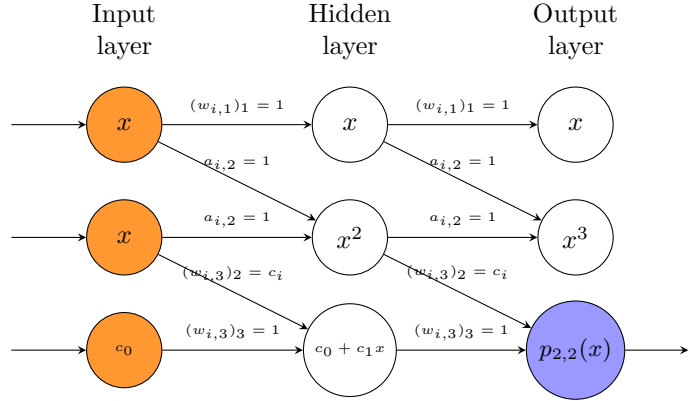

\analyticApprx*

\begin{proof}
Let $M\geq 1$, $s>2$, $C_K >0$, $\epsilon \in (0,1)$ and let $K$ be an analytic function with the required assumptions. As a first step, we approximate $K$ with a polynomial $h_n$ of degree $n\geq2$ (to be defined later). Then, we realize $h_n$ with a deep multiplicative network using Lemma~\ref{lemma:poly_realization}. We recall the polynomial $h_n$ described in Theorem~\ref{thm:chebychev_unique}. For any integer $n\geq 2$, we have
\[
\norm{h_n -K}_{L^{\infty} \prn{\brk{-M, M}}} ~\leq~ \fr{C_K s^{-n}}{s-1} ~=~ \cO\prn{C_K s^{-n}}.
\]
By choosing {$n=\frac{1}{\log_2 s} \log_2 \frac{C_K}{\epsilon\prn{s-1}} \leq \frac{1}{\log_2 s} \log_2 \frac{C_K}{\epsilon}$ } we obtain the following relation:
 
\begin{equation}
\begin{aligned}
\frac{\epsilon \prn{s-1}}{C_K} ~&=~ 2^{\log_2\prn{\frac{\epsilon \prn{s-1}}{C_K}}} \\
~&=~ 2^{-\frac{1}{\log_2 s}\prn{\log_2\prn{\frac{C_K}{\epsilon \prn{s-1}}}}\log_2 s} \\
~&=~ s^{-\frac{1}{\log_2 s}\prn{\log_2\prn{\frac{C_K}{\epsilon \prn{s-1}}}}} ~=~ s^{-n}.
\end{aligned}
\end{equation}
In particular,
\[
\norm{h_n -K}_{L^{\infty} \prn{\brk{-M, M}}} ~\leq~ \fr{C_K s^{-n}}{s-1} ~\leq~ \epsilon,
\]
for $s\geq 1$. Given that $h_n$ is a polynomial of degree $n$, by Lemma~\ref{lemma:poly_realization}, there exists $\nn{POL}$ and $c_0 = h_{n, 0}$ such that  $h_n \prn{x} =\nn{POL}\prn{x, x, c_0}$, and achieve the desired result.
\end{proof}

Let us now recall Maurey's Theorem ~\citep{SAF_1980-1981____A5_0,Milman1986} which will assist us in the proof of Theorem \ref{thm:bandlimitedApprx}.

\begin{lemma}[Maurey's theorem]\label{lem:maurey}
Let $\mathcal{V}$ be a Hilbert space with norm $\norm{\cdot}_{\mathcal{V}}$. Suppose there exists $Q\subset \mathcal{V}$ such that for every $q\in Q$, $\norm{q}_{\mathcal{V}}\leq b$  for some $b>0$. Then, for every $f$ in the convex hull of $Q$ and every integer $n\geq 1$, there exists a $f_n$ in the convex hull of $n$ points in $Q$ and a constant $c>b^2-\norm{f}_{\mathcal{V}}^2$ such that $\norm{f_n -f}_{\mathcal{V}}^2 \leq \frac{c}{n}$.
\end{lemma}

\bandlimitedApprx*

\begin{proof} Let $f\in \cH_{K, M}\prn{B}$. Further, let $F\prn{\omega}=\abs{F\prn{\omega}} \cdot \exp(i \theta\prn{\omega})$, the polar representation of $F\prn{\omega}$. The following holds:

\begin{equation}
\begin{aligned}
f\prn{x} ~&=~ \text{Re}\prn*{\int_{\brk{-M, M}^d} F\prn{\omega} K\prn{\omega \cdot x}~\textnormal{d}\omega} \\ 
~&=~ \text{Re}\prn*{\int_{\brk{-M, M}^d} C_F \exp\prn{i \theta\prn{\omega}} K\prn{\omega \cdot x} \frac{\abs{F\prn{\omega}}}{C_F}~\textnormal{d}\omega} \\ 
~&=~ \int_{\brk{-M, M}^d} {C_F \brk*{ \cos\prn{\theta\prn{\omega}} K_R \prn{\omega \cdot x} - \sin\prn{\theta\prn{\omega}} K_I \prn{\omega \cdot x} } \frac{\abs{F\prn{\omega}}}{C_F}~\textnormal{d}\omega},
\end{aligned}
\end{equation}

where $K_R, K_I$ are the real and imaginary parts of $K$ respectively. Given that the integral represents $f$ as an infinite convex combination of functions in 
\begin{equation*}
Q_{K, M} = \left\{\gamma\brk*{\cos\prn{\beta}K_R {\prn{\omega\cdot x}} - \sin\prn{\beta} K_I \prn{\omega \cdot x}} \mathrel{\Big|} \abs{\gamma}\leq C_F, \beta\in\R, \omega\in\brk{-M, M}^d\right\},
\end{equation*}

then $f$ is in the closure of the convex hull of $Q_{K, M}$. Due to the fact that $x\in\brk{0, 1}^d$ and $\omega\in\brk{-M, M}^d$, $t=\omega\cdot x \in \brk{-dM, dM}$. By the definition of $D_K$, functions in $Q_{K, M}$ are bounded in the $L^2\prn{B}$-norm by $2 C_F D_K\leq 2 C_F $. Using Lemma~\ref{lem:maurey}, there exist real coefficients $\crl{b_j}$ and $\crl{\beta_j}$, and vectors $\omega_j \in \brk{-M, M}^d$ for $1\leq j\leq \ceil{1/\epsilon_0^2}$, such that: 
\[
f_{\epsilon_0} \prn{x} = \sum_{j=1}^{\ceil{1/\epsilon_0^2}} {b_j \brk*{\cos\prn{\beta_j} K_R \prn{\omega_j \cdot x}- \sin\prn{\beta_j} K_I \prn{\omega_j \cdot x}}}, \quad \sum_{j=1}^{\ceil{1/\epsilon_0^2}} {\abs{b_j}} \leq C_F,
\]

for $0<\epsilon_0 < 1$ to be defined at a later time, such that 
\[
\norm{f_{\epsilon_0} \prn{x} - f\prn{x}}_{L^2 \prn{B}} \leq 2 C_F \epsilon_0.
\]

We are now ready to approximate $f_{\epsilon_0} \prn{x}$ using a deep multiplicative neural network $\nn{MBL}$ on $B$. We notice that $K_R \prn{x}$ and $K_I \prn{x}$ are analytic kernels that hold the assumptions of Theorem \ref{thm:deep_som_for_analytic}. They can therefore be approximated to $\epsilon_0$ error using networks $\nn{RMA}$, $\nn{IMA}$ of depth and number of neurons
\[
G_{\epsilon_0} \sim L_{\epsilon_0} = \cO\prn{\frac{1}{\log_2 s} \log_2 \frac{C_K}{\epsilon_0}}, 
\]
where RMA stands for `Real Multiplicative Analytic' and IMA stands for `Imaginary Multiplicative Analytic'. Let us define the multiplicative network $\nn{MBL}\prn{x}$ by
\[
\nn{MBL}\prn{x} = \sum_{j=1}^{\ceil{1/\epsilon_0^2}} {b_j \brk*{\cos \prn{\beta_j} {\nn{RMA}} \prn{\omega_j \cdot x}- \sin\prn{\beta_j} \nn{IMA} \prn{\omega_j \cdot x}}}.
\]
This network has depth $L_{\epsilon_0}=\cO\prn{\frac{1}{\log_2 s} \log_2 \frac{C_K}{\epsilon_0}}$ and $G_{\epsilon_0}=\cO\prn{\frac{1}{{\epsilon_0^2} \log_2 s} \log_2 \frac{C_K}{\epsilon_0}}$ neurons.

\begin{equation*}
\begin{aligned}
\norm{\nn{MBL}\prn{x}- f_{\epsilon_0}\prn{x}}_{L^\infty \prn{B}} ~&\leq~ \sum_{j=1}^{\ceil{1/\epsilon_0^2}} {\abs{b_j}\norm{\nn{RMA} \prn{\omega_j \cdot x}- 
K_R \prn{\omega_j \cdot x}}_{L^\infty \prn{B}}} \\
&\quad+ \sum_{j=1}^{\ceil{1/\epsilon_0^2}} {\abs{b_j}\norm{\nn{IMA} \prn{\omega_j \cdot x}- K_I \prn{\omega_j \cdot x}}_{L^\infty \prn{B}}} \\ ~&\leq~ 2 C_F \epsilon_0,
\end{aligned}
\end{equation*}
that implies 
\[
\norm{\nn{MBL} \prn{x} -f(x)}_{L^2 \prn{B}} \leq \norm{\nn{MBL} \prn{x} - f_{\epsilon_0}\prn{x}}_{L^2 \prn{B}} + \norm{f_{\epsilon_0} \prn{x} - f\prn{x}}_{L^2 \prn{B}} \leq 4 C_F {\epsilon_0}, 
\]
where the last inequality is due to the fact that 
\[
\norm{g}_{L^2 \prn{B}}= \prn*{\int_{B}\abs{g}^2}^{1/2} \leq \norm{g}_{L^\infty \prn{B}} = \norm{g}_{L^\infty \prn{B}}.
\]

Taking $\epsilon_0 = \epsilon/\prn{4 C_F}$ achieves the sought result. 
\end{proof}

We further investigate how the constant $C_K$ from Theorem \ref{thm:bandlimitedApprx} behaves as functions of $M$ and $s$. Let $K\prn{x} = \exp\prn{i x}$ be an example kernel, $x\in\brk{-M, M}^d$. We notice that $a^{ M}_s =  M(s+s^{-1})/2$ for $s>2$, is the larger axis, and therefore the maximal norm of $K$ on $E^{ M}_s$ is given by $K(a^{M}_s)$:
{\[
\max_{x} \abs{K\prn{x}} \leq \exp\prn{M \fr{s+s^{-1}}{2}} = C_K (s, M).
\]}

In our setting, inputs are in the interval $t = \omega\cdot x \in\brk{-d M, d M}$, so we may use the bounding constant
\[
\max_{t} \abs{K\prn{t}} ~\leq~ \exp\prn{dM \fr{s+s^{-1}}{2}} = C_K (s, dM).
\]
The resulting network $\nn{MBL}$ then has depth 
\[
L_{\epsilon} ~=~ \cO\prn*{\frac{1}{\log_2 s}\prn*{dM \frac{s+s^{-1}}{2} + \log_2 \frac{C_F }{\epsilon}}}
\]
and
\[
G_{\epsilon} = \cO\prn*{\frac{C_F^2 }{\epsilon^2 \log_2 s}\prn*{dM \frac{s+s^{-1}}{2} + \log_2 \frac{C_F }{\epsilon}}}
\]
neurons. In this scenario, we see that both a large band $M$ and a large dimension $d$ will linearly affect the first term. 

\strictSobRS*

\begin{proof} Let $f\in\mathcal{B}_{2r, 2}$, and $M>1$. Similar to the proof of Theorem \ref{thm:strict_sobolev_ms}, we define the bandlimiting of $f:\R^d\rightarrow\R$ by
\[
f_M ~=~ \mathcal{F}^{-1} \prn{\mathcal{F} {f}\mathbbm{1}_{\brk{-M, M}^d}}.
\]
Let us define $F:\brk{-M, M}^d \rightarrow \C$:
\[
F\prn{\omega} ~=~ \normalize \prn{\mathcal{F} f}\prn{\omega}, 
\]
and $K\prn{x} = \exp(i x)$. We then have the following identity:
\[
f_M (x) ~=~ \int_{\brk{-M, M}^d} F \prn{\omega}K\prn{\omega \cdot x} ~\textnormal{d}\omega ~=~ \int_{\brk{-M, M}^d} \fr{1}{(2\pi)^d} \mathcal{F}f \prn{\omega}  \exp\prn{i\omega\cdot x} ~\textnormal{d}\omega.
\]

It is then easy to see that $f_M \in {\cH_{K, M}\prn{B}}$. We choose $M= \prn{2/\epsilon}^{1/r}$. Using the same derivations as in the proof of Theorem \ref{thm:strict_sobolev_ms}, we seek to approximate $f_M$ with a network $\nn{RS}$ such that 
\[
\norm{f_M - \nn{RS}}_{L_2 \prn{B}} ~\leq~ \epsilon/2.
\]
We then arrive at
\begin{equation}
\begin{aligned}
\norm{f - \nn{RS}}_{L_2 \prn{B}} 
~&\leq~ \norm{f - f_M}_{L_2 \prn{B}} + \norm{f_M - \nn{RS}}_{L_2 \prn{B}} \\ 
~&\leq~ M^{-r} + \epsilon/2 ~\leq~ \epsilon.
\end{aligned}
\end{equation}
We consider that {$C_{F} = \int_{\brk{-M, M}^d} \abs{F\prn{\omega}}~\textnormal{d}\omega = \normalize\int_{\brk{-M, M}^d} \abs{\cF{f}\prn{\omega}}~\textnormal{d}\omega \leq 1$}, where the inequality is due to the definition of $\mathcal{B}_{2r, 2}$. The kernel {$K\prn{t} = \exp\prn{i t}$} takes as input $t=\omega \cdot x$, for $x \in B$ and $\omega\in\brk{-M, M}^d$. Therefore, $t\in\brk{-d M, d M}$, and $K:\brk{-d M, d M}\rightarrow\R$. We note that $K$ is continuable to the Bernstein 4-ellipse $E_4^{d M}$. We notice that $a^{d M}_4 = d M(4+4^{-1})/2$ is the larger axis, and therefore the maximal norm of $K$ on $E^{d M}_s$ is given by $K(a^{dM}_4)$:
\[
C_K ~=~ \max_{t} \abs{K\prn{t}} ~\leq~ \exp\prn{dM \fr{4+4^{-1}}{2}}.
\]
Further, for any $t\in\R$ we have 
 $\abs{K(t)}\leq 1 = D_K$.

Using Theorem 3.2 from \citep{Montanelli2021DeepRN} we can construct a deep ReLU network $\nn{RS}$ such that 
\[
\norm{f_M - \nn{RS}}_{L_2 \prn{B}} ~\leq~ \epsilon/2
\]
whose depth is
\begin{equation}
\begin{aligned}
L_{\epsilon} ~&\leq~ C_1 \frac{1}{\log_2^2 4} \log_2^2 {\frac{2 C_F C_K }{\epsilon}} \\ ~&\leq~ C_2 \frac{1}{\log_2^2 4} \prn*{\log_2 {\frac{ {\exp\prn{d M \frac{4 + 4^{-1}}{2}}}}{\epsilon}}}^2 \\ ~&\leq~ 12 C_2 \prn*{d \prn*{\frac{2}{\epsilon}}^{\frac{1}{r}}  - \log_2 {\frac{1}{\epsilon}}}^2 \\ ~&\leq~ C_3 d^2 \epsilon^{-\frac{2}{r}},
\end{aligned}
\end{equation}
for some constants $C_1, C_2, C_3 > 0$. In addition,  the number of neurons can be bounded by 
\begin{equation}
\begin{aligned}
G_{\epsilon} ~&\leq~ C'_1 {C_F^2} \frac{1}{\epsilon^2 \log_2^2 4} \log_2^2 {\frac{2 C_K C_F}{\epsilon}} \\ 
~&\leq~ C'_2 \frac{1}{\epsilon^2 \log_2^2 4} \prn*{\log_2 {\frac{ {\exp\prn{d M \frac{4 + 4^{-1}}{2}}}}{\epsilon}}}^2 \\ ~&\leq~   \frac{12 C'_2}{\epsilon^2} \prn*{d \prn*{\frac{2}{\epsilon}}^{\frac{1}{r}}  - \log_2 {\frac{1}{\epsilon}}}^2 \\ 
~&\leq~ C'_3 d^2 \epsilon^{-\prn{2 + \frac{2}{r}}},
\end{aligned}
\end{equation}
for some constants $C'_1, C'_2, C'_3 > 0$.
\end{proof}


\end{document}


\maketitle

\appendix

\section{Additional Experiments}\label{app:additional_experiments}


{\bf Auxiliary experiments on the effective depth.\enspace} We repeated the experiment in Fig.~1 in the main text. In Figs.~\ref{fig:cifar10_convnet_400_ext}-\ref{fig:fashion_mlp_100} we plot the results of the same experiment, with different networks and datasets (see captions). As can be seen, in all cases, for networks deeper than a threshold we obtain (near perfect) NCC separability in all of the top layers. Furthermore, similar to the results in Fig.~1 in the main text, the degree of neural collapse improves with the network's depth.

{\bf Auxiliary experiments with noisy labels.\enspace} We repeated the experiment in Figs.~1-2 in the main text. In Figs.~\ref{fig:cifar10_noisy_mlp} and~\ref{fig:fashion_noisy_conv} we plot the results of the same experiment, with different networks and datasets (see captions). As can be seen, the effective NCC depth of a neural network tends to increase as we train with increasing amounts of corrupted labels.

{\bf The effect of the width on intermediate collapse.\enspace} As an additional experiment, we studied the effect of the width on intermediate neural collapse. In Fig.~\ref{fig:mnist_convnet_depth_10} we plot the results of this experiment, for CONV-10-$H$ networks, with $H=20,40,80,160,320$. In each row we consider a different evaluation metric (the CDNV on the train and test data and the NCC classification accuracy on the train and test data) and in each column, we consider a neural network of a different width. As can be seen,
intermediate neural collapse strengthens when increasing the width of the neural network, on both train and test data.

\subsection{Estimating the Generalization Bound}\label{sec:est}

In Prop.~\ref{prop:bound} we introduce a generalization bound for deep neural networks. In this section we empirically estimate the bound and demonstrate non-trivial estimations of the test performance.

The results are summarized in Tab.~\ref{tab:GenBoundEst}. We report (an estimation of) the mean test error, the choices of $\epsilon$ and $p$, the estimations of the first term in Eq.~\ref{eq:bound} and the full bound in Eq.~\ref{eq:bound}. We experiment with multiple values of $p$ depending on the hardness of the given task. We chose $\epsilon=0.005$ as our default threshold for deciding whether we have separation or not.

{\bf Estimating the bound.\enspace} We would like to estimate the first term in the bound,
\begin{equation}
\P_{S_1,S_2,\tY_2}\left[\E_{\gamma}[\cd^{\epsilon}_{S_1}(h^{\gamma}_{S_1})] \geq \cd^{\epsilon}_{\min}(\cG,S_1\cup \tS_2)\right].   
\end{equation}
According to Prop.~\ref{prop:bound2} in order to estimate this term we need to generate i.i.d. triplets $(S^i_1,S^i_2,\tY^i_2)$. Since we have a limited access to training data, we use a variation of cross-validation and generate $k_1=5$ i.i.d. disjoint splits $(S^i_1,S^i_2)$ of the training data $S$. For each one of these pairs, we generate $k_2=3$ corrupted labelings $\tY^{ij}_2$. We denote by $\tS^{ij}_2$ the set obtained by replacing the labels of $S^i_2$ with $\tY^{ij}_2$ and $\tS^{ij}_{3} := S^i_1 \cup \tS^{ij}_2$. 

As a first step, we would like to estimate $\E_{\gamma}[\cd^{\epsilon}_{S^i_1}(h^{\gamma}_{S^i_1})]$ for each $i \in [k_1]$. For this purpose, we randomly select $T_1=5$ different initializations $\gamma_1,\dots,\gamma_{T_1}$ and for each one, we train the model $h^{\gamma_t}_{S^i_1}$ using the training protocol described in Sec.~4.1. Once trained, we compute $\cd^{\epsilon}_{S_1}(h^{\gamma_t}_{S^i_1})$ for each $t \in [T_1]$ (see Def.~1) and approximate $\E_{\gamma}[\cd^{\epsilon}_{S^i_1}(h^{\gamma}_{S^i_1})]$ using $d_i := \frac{1}{T_1} \sum^{T_1}_{t=1}\cd^{\epsilon}_{S^i_1}(h^{\gamma_t}_{S^i_1})$. 

As a next step, we would like to evaluate $\bI[d_i \geq \cd^{\epsilon}_{\min}(\cG, \tS^{ij}_3)]$. We notice that $d_i \geq \cd^{\epsilon}_{\min}(\cG, S^i_1 \cup \tS^i_2)$ if and only if there is a $d_i$-layered neural network $f = g^{d_i} \circ \dots \circ g^1$ for which $\textnormal{err}_{\tS^{ij}_3}(\hat{h}) \leq \epsilon$, where $\hat{h}(x) := \argmin_{c \in [C]} \|f(x)-\mu_{f}(S_c)\|$. In general, computing this boolean value is computationally hard. Therefore, to estimate this boolean value, we simply train a $(d_i+1)$-layered network $h = e \circ f$ and check whether its penultimate layer is $\epsilon$-NCC separable, i.e., $\textnormal{err}_{\tS^{ij}_3}(\hat{h}) \leq \epsilon$, where $\hat{h}(x) := \argmin_{c \in [C]} \|f(x)-\mu_{f}(S_c)\|$. If SGD implicitly optimizes neural networks to maximize NCC separability as observed in~\citep{Papyan24652} (and also in this paper), we should expect to obtain $\epsilon$-NCC separability in the penultimate layer if that is possible with a $d_i$-layered network. Since training might be non-optimal, to obtain a robust estimation, we train $T_2=5$ models $h_t = e_t \circ f_t$ of depth $d_i+1$ and pick the one with the best NCC separability in its penultimate layer. Namely, we replace $\cd^{\epsilon}_{\min}(\cG, \tS^{ij}_3)$ with $\min_{t \in [T_2]} \cd^{\epsilon}_{\tS^{ij}_3}(h_t)$ and estimate $\bI [d_i \geq \cd^{\epsilon}_{\min}(\cG, \tS^{ij}_3)]$ using $\bI[d_i \geq \min_{t \in [T_2]} \cd^{\epsilon}_{\tS^{ij}_3}(h_t)]$.

Our final estimation is the following
\begin{equation}
\label{eqn:prob_term_estimate}
\frac{1}{k_1}\sum^{k_1}_{i=1}\frac{1}{k_2} \sum^{k_2}_{j=1}\bI\left[d_i \geq \min_{t \in [T_2]} \cd^{\epsilon}_{\tS^{ij}_3}(h_t)\right] \approx \P_{S_1,S_2,\tY_2}\left[\E_{\gamma}[\cd^{\epsilon}_{S_1}(h^{\gamma}_{S_1})] \geq \cd^{\epsilon}_{\min}(\cG,S_1\cup \tS_2)\right].
\end{equation}
In order to estimate the bound we assume that $\delta^1_{m}$ and $\delta^{2}_{m,p,\alpha}$ are negligible constants and that $\alpha=1$. The estimation of the bound is given by the sum of the LHS in Eq.~\ref{eqn:prob_term_estimate} and $p$.

{\bf Estimating the mean test error.\enspace} To estimate the mean test error, $\E_{S_1,\gamma}[\textnormal{err}_{P}(h^{\gamma}_{S_1})]$, as typically done in machine learning, we replace the population distribution $P$ with the test set $S_{test}$ and we replace the expectation over $S_1$ and $\gamma$ with averages across the $k_1=5$ random selections of $\{S^i_1\}^{k_1}_{i=1}$ and $T_1=5$ random selections of $\{\gamma_t\}^{T_1}_{t=1}$. Namely, we compute the following $\frac{1}{k_1} \sum^{k_1}_{i=1} \frac{1}{T_1} \sum^{T_1}_{t=1} \textnormal{err}_{S_{test}}(h^{\gamma_t}_{S^i_1}) \approx \E_{S_1,\gamma}[\textnormal{err}_{P}(h^{\gamma}_{S_1})]$.

\begin{table}
  \centering
  \renewcommand{\arraystretch}{1.2}
  \begin{tabular}{ccccccccccc}
  \hline
    Dataset & \multicolumn{3}{c}{MNIST} & \multicolumn{3}{c}{Fashion MNIST} & \multicolumn{3}{c}{CIFAR10} \\
    \hline 
    Architecture & \multicolumn{3}{c}{CONV-10-50} & \multicolumn{3}{c}{CONV-10-100} &
    \multicolumn{3}{c}{CONV-16-100} \\
    $\E_{S_1,\gamma}[\textnormal{err}_{P}(h^{\gamma}_{S_1})]$ & \multicolumn{3}{c}{0.0075} & \multicolumn{3}{c}{0.0996} & \multicolumn{3}{c}{0.2676} \\
    \hline
    $\epsilon$ & \multicolumn{3}{c}{0.005} & \multicolumn{3}{c}{0.005} & \multicolumn{3}{c}{0.005} \\ 
    $\E_i[d_i], \sigma(d_i)$ & \multicolumn{3}{c}{5.91, 0.434} & \multicolumn{3}{c}{6.64, 0.344} & \multicolumn{3}{c}{6.87, 0.34} \\
    $p$ & 0.05 & 0.075 & 0.1 & 0.05 & 0.15 & 0.2 & 0.4 & 0.45 & 0.5 \\
    The estimation in Eq.~\ref{eqn:prob_term_estimate} & 1.0 & 0.4 & 0.0 & 1.0 & 0.6 & 0.0 & 0.27 & 0.27 & 0.2 \\
    Bound & 1.05 & 0.475 & {\bf 0.1} & 1.05 & 0.75 & {\bf 0.2} & {\bf 0.66} & 0.72 & 0.7 \\
    \hline
  \end{tabular}
  \caption{{\bf Estimating the bound in Prop.~\ref{prop:bound}.} See Sec.~\ref{sec:est} for details.}
  \label{tab:GenBoundEst}
\end{table}

\begin{figure*}[t]
  \centering
  \begin{tabular}{c@{}c@{}c@{}c@{}c}
    \hline \\
     && {\small\bf CDNV - Train} && \\
     \includegraphics[width=0.2\linewidth]{plots/cifar10/convnet_width_400/cdnv_train_3_0.pdf}  & \includegraphics[width=0.2\linewidth]{plots/cifar10/convnet_width_400/cdnv_train_4_0.pdf}  & \includegraphics[width=0.2\linewidth]{plots/cifar10/convnet_width_400/cdnv_train_5_0.pdf}  &
     \includegraphics[width=0.2\linewidth]{plots/cifar10/convnet_width_400/cdnv_train_8_0.pdf}  &
     \includegraphics[width=0.2\linewidth]{plots/cifar10/convnet_width_400/cdnv_train_10_0.pdf}
     \\
     {\small$3$ layers}  & {\small$4$ layers} & {\small$5$ layers} & {\small$8$ layers} & {\small$10$ layers} \\
     \includegraphics[width=0.2\linewidth]{plots/cifar10/convnet_width_400/cdnv_train_12_0.pdf}  & \includegraphics[width=0.2\linewidth]{plots/cifar10/convnet_width_400/cdnv_train_14_0.pdf}  & \includegraphics[width=0.2\linewidth]{plots/cifar10/convnet_width_400/cdnv_train_16_0.pdf}  &
     \includegraphics[width=0.2\linewidth]{plots/cifar10/convnet_width_400/cdnv_train_18_0.pdf}  &
     \includegraphics[width=0.2\linewidth]{plots/cifar10/convnet_width_400/cdnv_train_20_0.pdf}
     \\
     {\small$12$ layers}  & {\small$14$ layers} & {\small$16$ layers} & {\small$18$ layers} & {\small$20$ layers} \\
     \hline \\
     && {\small\bf NCC train accuracy} && \\
     \includegraphics[width=0.2\linewidth]{plots/cifar10/convnet_width_400/train_emb_acc_ncc_3_0.pdf}  & \includegraphics[width=0.2\linewidth]{plots/cifar10/convnet_width_400/train_emb_acc_ncc_4_0.pdf} & \includegraphics[width=0.2\linewidth]{plots/cifar10/convnet_width_400/train_emb_acc_ncc_5_0.pdf}  &
     \includegraphics[width=0.2\linewidth]{plots/cifar10/convnet_width_400/train_emb_acc_ncc_8.pdf}  &
     \includegraphics[width=0.2\linewidth]{plots/cifar10/convnet_width_400/train_emb_acc_ncc_10_0.pdf} 
     \\
     {\small$3$ layers}  & {\small$4$ layers} & {\small$5$ layers} & {\small$8$ layers} & {\small$10$ layers} \\
    \includegraphics[width=0.2\linewidth]{plots/cifar10/convnet_width_400/train_emb_acc_ncc_12_0.pdf}  & \includegraphics[width=0.2\linewidth]{plots/cifar10/convnet_width_400/train_emb_acc_ncc_14_0.pdf} & \includegraphics[width=0.2\linewidth]{plots/cifar10/convnet_width_400/train_emb_acc_ncc_16_0.pdf}  &
     \includegraphics[width=0.2\linewidth]{plots/cifar10/convnet_width_400/train_emb_acc_ncc_18_0.pdf}  &
     \includegraphics[width=0.2\linewidth]{plots/cifar10/convnet_width_400/train_emb_acc_ncc_20_0.pdf} 
     \\
     {\small$12$ layers}  & {\small$14$ layers} & {\small$16$ layers} & {\small$18$ layers} & {\small$20$ layers} \\
     \hline \\
     && {\small\bf CDNV - Test} && \\
  \includegraphics[width=0.2\linewidth]{plots/cifar10/convnet_width_400/cdnv_test_3_0.pdf}  & \includegraphics[width=0.2\linewidth]{plots/cifar10/convnet_width_400/cdnv_test_4_0.pdf}  & \includegraphics[width=0.2\linewidth]{plots/cifar10/convnet_width_400/cdnv_test_5_0.pdf}  &
     \includegraphics[width=0.2\linewidth]{plots/cifar10/convnet_width_400/cdnv_test_8_0.pdf}  &
     \includegraphics[width=0.2\linewidth]{plots/cifar10/convnet_width_400/cdnv_test_10_0.pdf}
     \\
     {\small$3$ layers}  & {\small$4$ layers} & {\small$5$ layers} & {\small$8$ layers} & {\small$10$ layers} \\
     \includegraphics[width=0.2\linewidth]{plots/cifar10/convnet_width_400/cdnv_test_12_0.pdf}  & \includegraphics[width=0.2\linewidth]{plots/cifar10/convnet_width_400/cdnv_test_14_0.pdf}  & \includegraphics[width=0.2\linewidth]{plots/cifar10/convnet_width_400/cdnv_test_16_0.pdf}  &
     \includegraphics[width=0.2\linewidth]{plots/cifar10/convnet_width_400/cdnv_test_18_0.pdf}  &
     \includegraphics[width=0.2\linewidth]{plots/cifar10/convnet_width_400/cdnv_test_20_0.pdf}
     \\
     {\small$12$ layers}  & {\small$14$ layers} & {\small$16$ layers} & {\small$18$ layers} & {\small$20$ layers} \\
     \hline \\
     && {\small\bf NCC test accuracy} && \\
     \includegraphics[width=0.2\linewidth]{plots/cifar10/convnet_width_400/test_emb_acc_ncc_3_0.pdf}  & \includegraphics[width=0.2\linewidth]{plots/cifar10/convnet_width_400/test_emb_acc_ncc_4_0.pdf} & \includegraphics[width=0.2\linewidth]{plots/cifar10/convnet_width_400/test_emb_acc_ncc_5_0.pdf}  &
     \includegraphics[width=0.2\linewidth]{plots/cifar10/convnet_width_400/test_emb_acc_ncc_8.pdf}  &
     \includegraphics[width=0.2\linewidth]{plots/cifar10/convnet_width_400/test_emb_acc_ncc_10_0.pdf} 
     \\
     {\small$3$ layers}  & {\small$4$ layers} & {\small$5$ layers} & {\small$8$ layers} & {\small$10$ layers} \\
    \includegraphics[width=0.2\linewidth]{plots/cifar10/convnet_width_400/test_emb_acc_ncc_12_0.pdf}  & \includegraphics[width=0.2\linewidth]{plots/cifar10/convnet_width_400/test_emb_acc_ncc_14_0.pdf} & \includegraphics[width=0.2\linewidth]{plots/cifar10/convnet_width_400/test_emb_acc_ncc_16_0.pdf}  &
     \includegraphics[width=0.2\linewidth]{plots/cifar10/convnet_width_400/test_emb_acc_ncc_18_0.pdf}  &
     \includegraphics[width=0.2\linewidth]{plots/cifar10/convnet_width_400/test_emb_acc_ncc_20_0.pdf} 
     \\
     {\small$12$ layers}  & {\small$14$ layers} & {\small$16$ layers} & {\small$18$ layers} & {\small$20$ layers} \\
     \hline
  \end{tabular}
  
 \caption{{\bf Intermediate neural collapse of CONV-$L$-400 trained on CIFAR10.} See Fig.~1 in the main text for details.}
 \label{fig:cifar10_convnet_400_ext}
\end{figure*}

\begin{figure*}[t]
  \centering
  \begin{tabular}{c@{}c@{}c@{}c@{}c}
    \hline \\
     && {\small\bf CDNV - Train} && \\
     \includegraphics[width=0.2\linewidth]{plots/cifar10/mlp_width_300/cdnv_train_3_0.pdf}  & \includegraphics[width=0.2\linewidth]{plots/cifar10/mlp_width_300/cdnv_train_4_0.pdf}  & \includegraphics[width=0.2\linewidth]{plots/cifar10/mlp_width_300/cdnv_train_5_0.pdf}  &
     \includegraphics[width=0.2\linewidth]{plots/cifar10/mlp_width_300/cdnv_train_8_0.pdf}  &
     \includegraphics[width=0.2\linewidth]{plots/cifar10/mlp_width_300/cdnv_train_10_0.pdf}
     \\
     {\small$3$ layers}  & {\small$4$ layers} & {\small$5$ layers} & {\small$8$ layers} & {\small$10$ layers} \\
     \includegraphics[width=0.2\linewidth]{plots/cifar10/mlp_width_300/cdnv_train_12_0.pdf}  & \includegraphics[width=0.2\linewidth]{plots/cifar10/mlp_width_300/cdnv_train_14_0.pdf}  & \includegraphics[width=0.2\linewidth]{plots/cifar10/mlp_width_300/cdnv_train_16_0.pdf}  &
     \includegraphics[width=0.2\linewidth]{plots/cifar10/mlp_width_300/cdnv_train_18_0.pdf}  &
     \includegraphics[width=0.2\linewidth]{plots/cifar10/mlp_width_300/cdnv_train_20_0.pdf}
     \\
     {\small$12$ layers}  & {\small$14$ layers} & {\small$16$ layers} & {\small$18$ layers} & {\small$20$ layers} \\
     \hline \\
     && {\small\bf NCC train accuracy} && \\
     \includegraphics[width=0.2\linewidth]{plots/cifar10/mlp_width_300/train_emb_acc_ccn_3_0.pdf}  & \includegraphics[width=0.2\linewidth]{plots/cifar10/mlp_width_300/train_emb_acc_ccn_4_0.pdf} & \includegraphics[width=0.2\linewidth]{plots/cifar10/mlp_width_300/train_emb_acc_ccn_5_0.pdf}  &
     \includegraphics[width=0.2\linewidth]{plots/cifar10/mlp_width_300/train_emb_acc_ccn_8_0.pdf}  &
     \includegraphics[width=0.2\linewidth]{plots/cifar10/mlp_width_300/train_emb_acc_ccn_10_0.pdf} 
     \\
     {\small$3$ layers}  & {\small$4$ layers} & {\small$5$ layers} & {\small$8$ layers} & {\small$10$ layers} \\
    \includegraphics[width=0.2\linewidth]{plots/cifar10/mlp_width_300/train_emb_acc_ccn_12_0.pdf}  & \includegraphics[width=0.2\linewidth]{plots/cifar10/mlp_width_300/train_emb_acc_ccn_14_0.pdf} & \includegraphics[width=0.2\linewidth]{plots/cifar10/mlp_width_300/train_emb_acc_ccn_16_0.pdf}  &
     \includegraphics[width=0.2\linewidth]{plots/cifar10/mlp_width_300/train_emb_acc_ccn_18_0.pdf}  &
     \includegraphics[width=0.2\linewidth]{plots/cifar10/mlp_width_300/train_emb_acc_ccn_20_0.pdf} 
     \\
     {\small$12$ layers}  & {\small$14$ layers} & {\small$16$ layers} & {\small$18$ layers} & {\small$20$ layers} \\
     \hline \\
     && {\small\bf CDNV - Test} && \\
  \includegraphics[width=0.2\linewidth]{plots/cifar10/mlp_width_300/cdnv_test_3_0.pdf}  & \includegraphics[width=0.2\linewidth]{plots/cifar10/mlp_width_300/cdnv_test_4_0.pdf}  & \includegraphics[width=0.2\linewidth]{plots/cifar10/mlp_width_300/cdnv_test_5_0.pdf}  &
     \includegraphics[width=0.2\linewidth]{plots/cifar10/mlp_width_300/cdnv_test_8_0.pdf}  &
     \includegraphics[width=0.2\linewidth]{plots/cifar10/mlp_width_300/cdnv_test_10_0.pdf}
     \\
     {\small$3$ layers}  & {\small$4$ layers} & {\small$5$ layers} & {\small$8$ layers} & {\small$10$ layers} \\
     \includegraphics[width=0.2\linewidth]{plots/cifar10/mlp_width_300/cdnv_test_12_0.pdf}  & \includegraphics[width=0.2\linewidth]{plots/cifar10/mlp_width_300/cdnv_test_14_0.pdf}  & \includegraphics[width=0.2\linewidth]{plots/cifar10/mlp_width_300/cdnv_test_16_0.pdf}  &
     \includegraphics[width=0.2\linewidth]{plots/cifar10/mlp_width_300/cdnv_test_18_0.pdf}  &
     \includegraphics[width=0.2\linewidth]{plots/cifar10/mlp_width_300/cdnv_test_20_0.pdf}
     \\
     {\small$12$ layers}  & {\small$14$ layers} & {\small$16$ layers} & {\small$18$ layers} & {\small$20$ layers} \\
     \hline \\
     && {\small\bf NCC test accuracy} && \\
     \includegraphics[width=0.2\linewidth]{plots/cifar10/mlp_width_300/test_emb_acc_ccn_3_0.pdf}  & \includegraphics[width=0.2\linewidth]{plots/cifar10/mlp_width_300/test_emb_acc_ccn_4_0.pdf} & \includegraphics[width=0.2\linewidth]{plots/cifar10/mlp_width_300/test_emb_acc_ccn_5_0.pdf}  &
     \includegraphics[width=0.2\linewidth]{plots/cifar10/mlp_width_300/test_emb_acc_ccn_8_0.pdf}  &
     \includegraphics[width=0.2\linewidth]{plots/cifar10/mlp_width_300/test_emb_acc_ccn_10_0.pdf} 
     \\
     {\small$3$ layers}  & {\small$4$ layers} & {\small$5$ layers} & {\small$8$ layers} & {\small$10$ layers} \\
    \includegraphics[width=0.2\linewidth]{plots/cifar10/mlp_width_300/test_emb_acc_ccn_12_0.pdf}  & \includegraphics[width=0.2\linewidth]{plots/cifar10/mlp_width_300/test_emb_acc_ccn_14_0.pdf} & \includegraphics[width=0.2\linewidth]{plots/cifar10/mlp_width_300/test_emb_acc_ccn_16_0.pdf}  &
     \includegraphics[width=0.2\linewidth]{plots/cifar10/mlp_width_300/test_emb_acc_ccn_18_0.pdf}  &
     \includegraphics[width=0.2\linewidth]{plots/cifar10/mlp_width_300/test_emb_acc_ccn_20_0.pdf} 
     \\
     {\small$12$ layers}  & {\small$14$ layers} & {\small$16$ layers} & {\small$18$ layers} & {\small$20$ layers} \\
     \hline
  \end{tabular}
  
 \caption{{\bf Intermediate neural collapse of MLP-$L$-300 trained on CIFAR10.} See Fig.~1 in the main text for details.}
 \label{fig:cifar10_mlp_300}
\end{figure*} 

\begin{figure*}[t]
  \centering
  \begin{tabular}{c@{}c@{}c@{}c@{}c}
     \hline \\
     && {\bf\small CDNV - Train} && \\
     \includegraphics[width=0.2\linewidth]{plots/mnist/convnet_width_50/cdnv_train_3_0.pdf}  &
     \includegraphics[width=0.2\linewidth]{plots/mnist/convnet_width_50/cdnv_train_4_0.pdf}  & \includegraphics[width=0.2\linewidth]{plots/mnist/convnet_width_50/cdnv_train_5_0.pdf}  &
     \includegraphics[width=0.2\linewidth]{plots/mnist/convnet_width_50/cdnv_train_8_0.pdf}  &
     \includegraphics[width=0.2\linewidth]{plots/mnist/convnet_width_50/cdnv_train_10_0.pdf}
     \\
     {\small$3$ layers}  & {\small$4$ layers} & {\small$5$ layers} & {\small$8$ layers} & {\small$10$ layers} \\
     \includegraphics[width=0.2\linewidth]{plots/mnist/convnet_width_50/cdnv_train_12_0.pdf}  &
     \includegraphics[width=0.2\linewidth]{plots/mnist/convnet_width_50/cdnv_train_14_0.pdf}  & \includegraphics[width=0.2\linewidth]{plots/mnist/convnet_width_50/cdnv_train_16_0.pdf}  &
     \includegraphics[width=0.2\linewidth]{plots/mnist/convnet_width_50/cdnv_train_18_0.pdf}  &
     \includegraphics[width=0.2\linewidth]{plots/mnist/convnet_width_50/cdnv_train_20_0.pdf}
     \\
     {\small$12$ layers}  & {\small$14$ layers} & {\small$16$ layers} & {\small$18$ layers} & {\small$20$ layers} \\
     \hline \\
     && {\bf\small NCC train accuracy} && \\
     \includegraphics[width=0.2\linewidth]{plots/mnist/convnet_width_50/train_emb_acc_ncc_3_0.pdf}  & \includegraphics[width=0.2\linewidth]{plots/mnist/convnet_width_50/train_emb_acc_ncc_4_0.pdf}  & \includegraphics[width=0.2\linewidth]{plots/mnist/convnet_width_50/train_emb_acc_ncc_5_0.pdf}  &
     \includegraphics[width=0.2\linewidth]{plots/mnist/convnet_width_50/train_emb_acc_ncc_8_0.pdf}  &
     \includegraphics[width=0.2\linewidth]{plots/mnist/convnet_width_50/train_emb_acc_ncc_10_0.pdf}
     \\
     {\small$3$ layers}  & {\small$4$ layers} & {\small$5$ layers} & {\small$8$ layers} & {\small$10$ layers} \\
     \includegraphics[width=0.2\linewidth]{plots/mnist/convnet_width_50/train_emb_acc_ncc_12_0.pdf}  & \includegraphics[width=0.2\linewidth]{plots/mnist/convnet_width_50/train_emb_acc_ncc_14_0.pdf}  & \includegraphics[width=0.2\linewidth]{plots/mnist/convnet_width_50/train_emb_acc_ncc_16_0.pdf}  &
     \includegraphics[width=0.2\linewidth]{plots/mnist/convnet_width_50/train_emb_acc_ncc_18_0.pdf}  &
     \includegraphics[width=0.2\linewidth]{plots/mnist/convnet_width_50/train_emb_acc_ncc_20_0.pdf}
     \\
     {\small$12$ layers}  & {\small$14$ layers} & {\small$16$ layers} & {\small$18$ layers} & {\small$20$ layers} \\
     \hline \\
     && {\bf\small CDNV - Test} && \\
     \includegraphics[width=0.2\linewidth]{plots/mnist/convnet_width_50/cdnv_test_3_0.pdf}  &
     \includegraphics[width=0.2\linewidth]{plots/mnist/convnet_width_50/cdnv_test_4_0.pdf}  & \includegraphics[width=0.2\linewidth]{plots/mnist/convnet_width_50/cdnv_test_5_0.pdf}  &
     \includegraphics[width=0.2\linewidth]{plots/mnist/convnet_width_50/cdnv_test_8_0.pdf}  &
     \includegraphics[width=0.2\linewidth]{plots/mnist/convnet_width_50/cdnv_test_10_0.pdf}
     \\
     {\small$3$ layers}  & {\small$4$ layers} & {\small$5$ layers} & {\small$8$ layers} & {\small$10$ layers} \\
     \includegraphics[width=0.2\linewidth]{plots/mnist/convnet_width_50/cdnv_test_12_0.pdf}  &
     \includegraphics[width=0.2\linewidth]{plots/mnist/convnet_width_50/cdnv_test_14_0.pdf}  & \includegraphics[width=0.2\linewidth]{plots/mnist/convnet_width_50/cdnv_test_16_0.pdf}  &
     \includegraphics[width=0.2\linewidth]{plots/mnist/convnet_width_50/cdnv_test_18_0.pdf}  &
     \includegraphics[width=0.2\linewidth]{plots/mnist/convnet_width_50/cdnv_test_20_0.pdf}
     \\
     {\small$12$ layers}  & {\small$14$ layers} & {\small$16$ layers} & {\small$18$ layers} & {\small$20$ layers} \\
     \hline \\
     && {\bf\small NCC test accuracy} && \\
     \includegraphics[width=0.2\linewidth]{plots/mnist/convnet_width_50/test_emb_acc_ncc_3_0.pdf}  & \includegraphics[width=0.2\linewidth]{plots/mnist/convnet_width_50/test_emb_acc_ncc_4_0.pdf}  & \includegraphics[width=0.2\linewidth]{plots/mnist/convnet_width_50/test_emb_acc_ncc_5_0.pdf}  &
     \includegraphics[width=0.2\linewidth]{plots/mnist/convnet_width_50/test_emb_acc_ncc_8_0.pdf}  &
     \includegraphics[width=0.2\linewidth]{plots/mnist/convnet_width_50/test_emb_acc_ncc_10_0.pdf}
     \\
     {\small$3$ layers}  & {\small$4$ layers} & {\small$5$ layers} & {\small$8$ layers} & {\small$10$ layers} \\
     \includegraphics[width=0.2\linewidth]{plots/mnist/convnet_width_50/test_emb_acc_ncc_12_0.pdf}  & \includegraphics[width=0.2\linewidth]{plots/mnist/convnet_width_50/test_emb_acc_ncc_14_0.pdf}  & \includegraphics[width=0.2\linewidth]{plots/mnist/convnet_width_50/test_emb_acc_ncc_16_0.pdf}  &
     \includegraphics[width=0.2\linewidth]{plots/mnist/convnet_width_50/test_emb_acc_ncc_18_0.pdf}  &
     \includegraphics[width=0.2\linewidth]{plots/mnist/convnet_width_50/test_emb_acc_ncc_20_0.pdf}
     \\
     {\small$12$ layers}  & {\small$14$ layers} & {\small$16$ layers} & {\small$18$ layers} & {\small$20$ layers} \\
     \hline
  \end{tabular}
 \caption{{\bf Intermediate neural collapse of CONV-$L$-50 trained on MNIST.} See Fig.~1 in the main text for details.}
 \label{fig:mnist_convnet_50}
\end{figure*}

\begin{figure*}[t]
  \centering
  \begin{tabular}{c@{}c@{}c@{}c@{}c}
  \hline \\
  && {\small\bf CDNV - Train} && \\
     \includegraphics[width=0.2\linewidth]{plots/fashion_mnist/mlp_width_100/cdnv_train_2_0.pdf}  & 
     \includegraphics[width=0.2\linewidth]{plots/fashion_mnist/mlp_width_100/cdnv_train_3_0.pdf}
     & 
    \includegraphics[width=0.2\linewidth]{plots/fashion_mnist/mlp_width_100/cdnv_train_4_0.pdf} 
    &
     \includegraphics[width=0.2\linewidth]{plots/fashion_mnist/mlp_width_100/cdnv_train_6_0.pdf}  &
     \includegraphics[width=0.2\linewidth]{plots/fashion_mnist/mlp_width_100/cdnv_train_8_0.pdf}
     \\
     {\small 2 hidden layers}  & {\small 3 hidden layers} & {\small 4 hidden layers} & {\small 6 hidden layers} & {\small 8 hidden layers} \\
     \includegraphics[width=0.2\linewidth]{plots/fashion_mnist/mlp_width_100/cdnv_train_10_0.pdf}  & 
     \includegraphics[width=0.2\linewidth]{plots/fashion_mnist/mlp_width_100/cdnv_train_12_0.pdf}
     & 
    \includegraphics[width=0.2\linewidth]{plots/fashion_mnist/mlp_width_100/cdnv_train_14_0.pdf} 
    &
     \includegraphics[width=0.2\linewidth]{plots/fashion_mnist/mlp_width_100/cdnv_train_16_0.pdf}  &
     \includegraphics[width=0.2\linewidth]{plots/fashion_mnist/mlp_width_100/cdnv_train_18_0.pdf}
     \\
     {\small 10 hidden layers} & {\small 12 hidden layers} & {\small 14 hidden layers} & {\small 16 hidden layers} & {\small 18 hidden layers} \\
     \hline \\
     && {\small\bf NCC train accuracy} && \\
     \includegraphics[width=0.2\linewidth]{plots/fashion_mnist/mlp_width_100/train_emb_acc_ncc_2_0.pdf}  & \includegraphics[width=0.2\linewidth]{plots/fashion_mnist/mlp_width_100/train_emb_acc_ncc_3_0.pdf}  & \includegraphics[width=0.2\linewidth]{plots/fashion_mnist/mlp_width_100/train_emb_acc_ncc_4_0.pdf}  &
     \includegraphics[width=0.2\linewidth]{plots/fashion_mnist/mlp_width_100/train_emb_acc_ncc_6_0.pdf}  &
     \includegraphics[width=0.2\linewidth]{plots/fashion_mnist/mlp_width_100/train_emb_acc_ncc_8_0.pdf}
     \\
     {\small 2 hidden layers} & {\small 3 hidden layers} & {\small 4 hidden layers} & {\small 6 hidden layers} & {\small 8 hidden layers} \\
     \includegraphics[width=0.2\linewidth]{plots/fashion_mnist/mlp_width_100/train_emb_acc_ncc_10_0.pdf}  & \includegraphics[width=0.2\linewidth]{plots/fashion_mnist/mlp_width_100/train_emb_acc_ncc_12_0.pdf}  & \includegraphics[width=0.2\linewidth]{plots/fashion_mnist/mlp_width_100/train_emb_acc_ncc_14_0.pdf}  &
     \includegraphics[width=0.2\linewidth]{plots/fashion_mnist/mlp_width_100/train_emb_acc_ncc_16_0.pdf}  &
     \includegraphics[width=0.2\linewidth]{plots/fashion_mnist/mlp_width_100/train_emb_acc_ncc_18_0.pdf}
     \\
     {\small 10 hidden layers} & {\small 12 hidden layers} & {\small 14 hidden layers} & {\small 16 hidden layers} & {\small 18 hidden layers} \\
  \hline \\
  && {\small\bf CDNV - Test} && \\
     \includegraphics[width=0.2\linewidth]{plots/fashion_mnist/mlp_width_100/cdnv_test_2_0.pdf}  & 
     \includegraphics[width=0.2\linewidth]{plots/fashion_mnist/mlp_width_100/cdnv_test_3_0.pdf}
     & 
    \includegraphics[width=0.2\linewidth]{plots/fashion_mnist/mlp_width_100/cdnv_test_4_0.pdf} 
    &
     \includegraphics[width=0.2\linewidth]{plots/fashion_mnist/mlp_width_100/cdnv_test_6_0.pdf}  &
     \includegraphics[width=0.2\linewidth]{plots/fashion_mnist/mlp_width_100/cdnv_test_8_0.pdf}
     \\
     {\small 2 hidden layers}  & {\small 3 hidden layers} & {\small 4 hidden layers} & {\small 6 hidden layers} & {\small 8 hidden layers} \\
     \includegraphics[width=0.2\linewidth]{plots/fashion_mnist/mlp_width_100/cdnv_test_10_0.pdf}  & 
     \includegraphics[width=0.2\linewidth]{plots/fashion_mnist/mlp_width_100/cdnv_test_12_0.pdf}
     & 
    \includegraphics[width=0.2\linewidth]{plots/fashion_mnist/mlp_width_100/cdnv_test_14_0.pdf} 
    &
     \includegraphics[width=0.2\linewidth]{plots/fashion_mnist/mlp_width_100/cdnv_test_16_0.pdf}  &
     \includegraphics[width=0.2\linewidth]{plots/fashion_mnist/mlp_width_100/cdnv_test_18_0.pdf}
     \\
     {\small 10 hidden layers} & {\small 12 hidden layers} & {\small 14 hidden layers} & {\small 16 hidden layers} & {\small 18 hidden layers} \\
     \hline \\
     && {\small\bf NCC test accuracy} && \\
     \includegraphics[width=0.2\linewidth]{plots/fashion_mnist/mlp_width_100/test_emb_acc_ncc_2_0.pdf}  & \includegraphics[width=0.2\linewidth]{plots/fashion_mnist/mlp_width_100/test_emb_acc_ncc_3_0.pdf}  & \includegraphics[width=0.2\linewidth]{plots/fashion_mnist/mlp_width_100/test_emb_acc_ncc_4_0.pdf}  &
     \includegraphics[width=0.2\linewidth]{plots/fashion_mnist/mlp_width_100/test_emb_acc_ncc_6_0.pdf}  &
     \includegraphics[width=0.2\linewidth]{plots/fashion_mnist/mlp_width_100/test_emb_acc_ncc_8_0.pdf}
     \\
     {\small 2 hidden layers} & {\small 3 hidden layers} & {\small 4 hidden layers} & {\small 6 hidden layers} & {\small 8 hidden layers} \\
     \includegraphics[width=0.2\linewidth]{plots/fashion_mnist/mlp_width_100/test_emb_acc_ncc_10_0.pdf}  & \includegraphics[width=0.2\linewidth]{plots/fashion_mnist/mlp_width_100/test_emb_acc_ncc_12_0.pdf}  & \includegraphics[width=0.2\linewidth]{plots/fashion_mnist/mlp_width_100/test_emb_acc_ncc_14_0.pdf}  &
     \includegraphics[width=0.2\linewidth]{plots/fashion_mnist/mlp_width_100/test_emb_acc_ncc_16_0.pdf}  &
     \includegraphics[width=0.2\linewidth]{plots/fashion_mnist/mlp_width_100/test_emb_acc_ncc_18_0.pdf}
     \\
     {\small 10 hidden layers} & {\small 12 hidden layers} & {\small 14 hidden layers} & {\small 16 hidden layers} & {\small 18 hidden layers} \\
     \hline
  \end{tabular}
  
 \caption{{\bf Intermediate neural collapse of MLP-$L$-100 trained on Fashion MNIST.} See Fig.~1 in the main text for details.} 
 \label{fig:fashion_mlp_100}
\end{figure*}

\begin{figure*}[t]
  \centering
  \begin{tabular}{c@{}c@{}c@{}c@{}c@{}c@{}c}
     \rotatebox[origin=c]{90}{{\tiny CDNV - Train}} &
     \raisebox{-0.5\height}{\includegraphics[width=0.2\linewidth]{plots/cifar10/mlp_width_500_noise/cdnv_train_varnoise_3.pdf}}  & 
     \raisebox{-0.5\height}{\includegraphics[width=0.2\linewidth]{plots/cifar10/mlp_width_500_noise/cdnv_train_varnoise_4.pdf}}  & 
     \raisebox{-0.5\height}{\includegraphics[width=0.2\linewidth]{plots/cifar10/mlp_width_500_noise/cdnv_train_varnoise_6.pdf}}  & 
     \raisebox{-0.5\height}{\includegraphics[width=0.2\linewidth]{plots/cifar10/mlp_width_500_noise/cdnv_train_varnoise_8.pdf}}  &
     \raisebox{-0.5\height}{\includegraphics[width=0.2\linewidth]{plots/cifar10/mlp_width_500_noise/cdnv_train_varnoise_10.pdf}} 
     \\
     & {\small layer 3} & {\small layer 4} & {\small layer 6} & {\small layer 8} & {\small layer 10} \\
     \rotatebox[origin=c]{90}{{\tiny NCC train acc}} &
     \raisebox{-0.5\height}{\includegraphics[width=0.2\linewidth]{plots/cifar10/mlp_width_500_noise/train_emb_acc_ncc_10_0.pdf}}    &
     \raisebox{-0.5\height}{\includegraphics[width=0.2\linewidth]{plots/cifar10/mlp_width_500_noise/train_emb_acc_ncc_10_10.pdf}}    & \raisebox{-0.5\height}{\includegraphics[width=0.2\linewidth]{plots/cifar10/mlp_width_500_noise/train_emb_acc_ncc_10_25.pdf}}  & \raisebox{-0.5\height}{\includegraphics[width=0.2\linewidth]{plots/cifar10/mlp_width_500_noise/train_emb_acc_ncc_10_50.pdf}}   &
     \raisebox{-0.5\height}{\includegraphics[width=0.2\linewidth]{plots/cifar10/mlp_width_500_noise/train_emb_acc_ncc_10_75.pdf}} 
     \\
     & {\small $0\%$ noise} & {\small $10\%$ noise}  & {\small $25\%$ noise} & {\small $50\%$ noise} & {\small $75\%$ noise} \\
     \rotatebox[origin=c]{90}{{\tiny CDNV - Test}} &
     \raisebox{-0.5\height}{\includegraphics[width=0.2\linewidth]{plots/cifar10/mlp_width_500_noise/cdnv_test_varnoise_3.pdf}}  & 
     \raisebox{-0.5\height}{\includegraphics[width=0.2\linewidth]{plots/cifar10/mlp_width_500_noise/cdnv_test_varnoise_4.pdf}}  & 
     \raisebox{-0.5\height}{\includegraphics[width=0.2\linewidth]{plots/cifar10/mlp_width_500_noise/cdnv_test_varnoise_6.pdf}}  & 
     \raisebox{-0.5\height}{\includegraphics[width=0.2\linewidth]{plots/cifar10/mlp_width_500_noise/cdnv_test_varnoise_8.pdf}}  &
     \raisebox{-0.5\height}{\includegraphics[width=0.2\linewidth]{plots/cifar10/mlp_width_500_noise/cdnv_test_varnoise_10.pdf}} 
     \\
     & {\small layer 3} & {\small layer 4} & {\small layer 6} & {\small layer 8} & {\small layer 10} \\
     \rotatebox[origin=c]{90}{{\tiny NCC test acc}} &
     \raisebox{-0.5\height}{\includegraphics[width=0.2\linewidth]{plots/cifar10/mlp_width_500_noise/test_emb_acc_ncc_10_0.pdf}}    & 
     \raisebox{-0.5\height}{\includegraphics[width=0.2\linewidth]{plots/cifar10/mlp_width_500_noise/test_emb_acc_ncc_10_10.pdf}}    & 
     \raisebox{-0.5\height}{\includegraphics[width=0.2\linewidth]{plots/cifar10/mlp_width_500_noise/test_emb_acc_ncc_10_25.pdf}}  & \raisebox{-0.5\height}{\includegraphics[width=0.2\linewidth]{plots/cifar10/mlp_width_500_noise/test_emb_acc_ncc_10_50.pdf}}   &
     \raisebox{-0.5\height}{\includegraphics[width=0.2\linewidth]{plots/cifar10/mlp_width_500_noise/test_emb_acc_ncc_10_75.pdf}}
     \\
     & {\small $0\%$ noise} & {\small $10\%$ noise}  & {\small $25\%$ noise} & {\small $50\%$ noise} & {\small $75\%$ noise} 
  \end{tabular}
  
 \caption{{\bf Intermediate neural collapse of MLP-10-500 trained on CIFAR10 with noisy labels.} See Fig.~2 in the main text for details.
 }
 \label{fig:cifar10_noisy_mlp}
\end{figure*} 

\begin{figure*}[t]
  \centering
  \begin{tabular}{c@{}c@{}c@{}c@{}c@{}c}
    \rotatebox[origin=c]{90}{{\tiny CDNV - Train}} &
     \raisebox{-0.5\height}{\includegraphics[width=0.2\linewidth]{plots/fashion_mnist/convnet_width_100_noise/cdnv_train_varnoise_4.pdf}}  &
     \raisebox{-0.5\height}{\includegraphics[width=0.2\linewidth]{plots/fashion_mnist/convnet_width_100_noise/cdnv_train_varnoise_6.pdf}}  & \raisebox{-0.5\height}{\includegraphics[width=0.2\linewidth]{plots/fashion_mnist/convnet_width_100_noise/cdnv_train_varnoise_8.pdf}}  & \raisebox{-0.5\height}{\includegraphics[width=0.2\linewidth]{plots/fashion_mnist/convnet_width_100_noise/cdnv_train_varnoise_9.pdf}}  &
     \raisebox{-0.5\height}{\includegraphics[width=0.2\linewidth]{plots/fashion_mnist/convnet_width_100_noise/cdnv_train_varnoise_10.pdf}} 
     \\
     & {\small layer 4} & {\small layer 6} & {\small layer 8} & {\small layer 9} &  {\small layer 10} \\
     \rotatebox[origin=c]{90}{{\tiny NCC train acc}} &
     \raisebox{-0.5\height}{\includegraphics[width=0.2\linewidth]{plots/fashion_mnist/convnet_width_100_noise/train_emb_acc_ccn_10_0.pdf}}  & \raisebox{-0.5\height}{\includegraphics[width=0.2\linewidth]{plots/fashion_mnist/convnet_width_100_noise/train_emb_acc_ccn_10_5.pdf}}  & \raisebox{-0.5\height}{\includegraphics[width=0.2\linewidth]{plots/fashion_mnist/convnet_width_100_noise/train_emb_acc_ccn_10_10.pdf}} &
     \raisebox{-0.5\height}{\includegraphics[width=0.2\linewidth]{plots/fashion_mnist/convnet_width_100_noise/train_emb_acc_ccn_10_20.pdf}} &
     \raisebox{-0.5\height}{\includegraphics[width=0.2\linewidth]{plots/fashion_mnist/convnet_width_100_noise/train_emb_acc_ccn_10_30.pdf}} 
     \\
     & {\small $0\%$ noise}  & {\small $5\%$ noise} & {\small $10\%$ noise} & {\small $20\%$ noise} & {\small $30\%$ noise} \\
     \rotatebox[origin=c]{90}{{\tiny CDNV - Test}} &
     \raisebox{-0.5\height}{\includegraphics[width=0.2\linewidth]{plots/fashion_mnist/convnet_width_100_noise/cdnv_test_varnoise_4.pdf}}  &
     \raisebox{-0.5\height}{\includegraphics[width=0.2\linewidth]{plots/fashion_mnist/convnet_width_100_noise/cdnv_test_varnoise_6.pdf}}  & \raisebox{-0.5\height}{\includegraphics[width=0.2\linewidth]{plots/fashion_mnist/convnet_width_100_noise/cdnv_test_varnoise_8.pdf}}  & \raisebox{-0.5\height}{\includegraphics[width=0.2\linewidth]{plots/fashion_mnist/convnet_width_100_noise/cdnv_test_varnoise_9.pdf}}  &
     \raisebox{-0.5\height}{\includegraphics[width=0.2\linewidth]{plots/fashion_mnist/convnet_width_100_noise/cdnv_test_varnoise_10.pdf}} 
     \\
     & {\small layer 4} & {\small layer 6} & {\small layer 8} & {\small layer 9} &  {\small layer 10} \\
     \rotatebox[origin=c]{90}{{\tiny NCC test acc}} &
     \raisebox{-0.5\height}{\includegraphics[width=0.2\linewidth]{plots/fashion_mnist/convnet_width_100_noise/test_emb_acc_ccn_10_0.pdf}}  & \raisebox{-0.5\height}{\includegraphics[width=0.2\linewidth]{plots/fashion_mnist/convnet_width_100_noise/test_emb_acc_ccn_10_5.pdf}}  & \raisebox{-0.5\height}{\includegraphics[width=0.2\linewidth]{plots/fashion_mnist/convnet_width_100_noise/test_emb_acc_ccn_10_10.pdf}}  &
     \raisebox{-0.5\height}{\includegraphics[width=0.2\linewidth]{plots/fashion_mnist/convnet_width_100_noise/test_emb_acc_ccn_10_20.pdf}}  &
     \raisebox{-0.5\height}{\includegraphics[width=0.2\linewidth]{plots/fashion_mnist/convnet_width_100_noise/test_emb_acc_ccn_10_30.pdf}} 
     \\
     & {\small $0\%$ noise}  & {\small $5\%$ noise} & {\small $10\%$ noise} & {\small $20\%$ noise} & {\small $30\%$ noise}
  \end{tabular}
  
 \caption{{\bf Intermediate neural collapse of CONV-10-100 trained on Fashion MNIST with noisy labels.} See Fig.~2 in the main text for details.}
 \label{fig:fashion_noisy_conv}
\end{figure*}

\begin{figure*}[t]
  \centering
  \begin{tabular}{c@{}c@{}c@{}c@{}c@{}c}
    \rotatebox[origin=c]{90}{{\tiny CDNV - Train}} &
     \raisebox{-0.5\height}{\includegraphics[width=0.2\linewidth]{plots/mnist/convnet_depth_10/cdnv_test_20_0.pdf}} & 
    \raisebox{-0.5\height}{\includegraphics[width=0.2\linewidth]{plots/mnist/convnet_depth_10/cdnv_test_40_0.pdf}} & 
    \raisebox{-0.5\height}{\includegraphics[width=0.2\linewidth]{plots/mnist/convnet_depth_10/cdnv_test_80_0.pdf}} &
    \raisebox{-0.5\height}{\includegraphics[width=0.2\linewidth]{plots/mnist/convnet_depth_10/cdnv_test_160_0.pdf}} &
    \raisebox{-0.5\height}{\includegraphics[width=0.2\linewidth]{plots/mnist/convnet_depth_10/cdnv_test_320_0.pdf}}
    \\
    \rotatebox[origin=c]{90}{{\tiny NCC train acc}} &
    \raisebox{-0.5\height}{\includegraphics[width=0.2\linewidth]{plots/mnist/convnet_depth_10/train_emb_acc_ncc_20_0.pdf}}  & 
    \raisebox{-0.5\height}{\includegraphics[width=0.2\linewidth]{plots/mnist/convnet_depth_10/train_emb_acc_ncc_40_0.pdf}}  & 
    \raisebox{-0.5\height}{\includegraphics[width=0.2\linewidth]{plots/mnist/convnet_depth_10/train_emb_acc_ncc_80_0.pdf}}  &
    \raisebox{-0.5\height}{\includegraphics[width=0.2\linewidth]{plots/mnist/convnet_depth_10/train_emb_acc_ncc_160_0.pdf}}  &
    \raisebox{-0.5\height}{\includegraphics[width=0.2\linewidth]{plots/mnist/convnet_depth_10/train_emb_acc_ncc_320_0.pdf}}
    \\
    \rotatebox[origin=c]{90}{{\tiny CDNV - Test}} &
    \raisebox{-0.5\height}{\includegraphics[width=0.2\linewidth]{plots/mnist/convnet_depth_10/cdnv_test_20_0.pdf}} & 
    \raisebox{-0.5\height}{\includegraphics[width=0.2\linewidth]{plots/mnist/convnet_depth_10/cdnv_test_40_0.pdf}} & 
    \raisebox{-0.5\height}{\includegraphics[width=0.2\linewidth]{plots/mnist/convnet_depth_10/cdnv_test_80_0.pdf}} &
    \raisebox{-0.5\height}{\includegraphics[width=0.2\linewidth]{plots/mnist/convnet_depth_10/cdnv_test_160_0.pdf}} &
    \raisebox{-0.5\height}{\includegraphics[width=0.2\linewidth]{plots/mnist/convnet_depth_10/cdnv_test_320_0.pdf}}
    \\
     \rotatebox[origin=c]{90}{{\tiny NCC test acc}} &
     \raisebox{-0.5\height}{\includegraphics[width=0.2\linewidth]{plots/mnist/convnet_depth_10/test_emb_acc_ncc_20_0.pdf}}  & 
     \raisebox{-0.5\height}{\includegraphics[width=0.2\linewidth]{plots/mnist/convnet_depth_10/test_emb_acc_ncc_40_0.pdf}}  & 
     \raisebox{-0.5\height}{\includegraphics[width=0.2\linewidth]{plots/mnist/convnet_depth_10/test_emb_acc_ncc_80_0.pdf}}  &
     \raisebox{-0.5\height}{\includegraphics[width=0.2\linewidth]{plots/mnist/convnet_depth_10/test_emb_acc_ncc_160_0.pdf}}  &
     \raisebox{-0.5\height}{\includegraphics[width=0.2\linewidth]{plots/mnist/convnet_depth_10/test_emb_acc_ncc_320_0.pdf}}
     \\
     & {\small width 20}  & {\small width 40} & {\small width 80} & {\small width 160} & {\small width 320} 
  \end{tabular}
  
 \caption{{\bf Intermediate neural collapse of CONV-10-$H$ trained on MNIST when varying the width.} See Fig.~1 in the main text for details.} 
 \label{fig:mnist_convnet_depth_10}
\end{figure*} 

\clearpage

\section{Proofs}

\begin{restatable}{proposition}{bound}\label{prop:bound} Let $m \in \mathbb{N}$, $p\in (0,1/2)$, $\alpha \in (0,1)$ and $\epsilon\in (0,1)$. Assume that the error of the learning algorithm is $\delta^1_m$-uniform. Assume that $S_1,S_2 \sim P_B(m)$. Let $h^{\gamma}_{S_1}$ be the output of the learning algorithm given access to a dataset $S_1$ and initialization $\gamma$. Then, 
\begin{equation}\label{eq:bound}
\begin{aligned}
\E_{S_1}\E_{\gamma}[\textnormal{err}_P(h^{\gamma}_{S_1})] ~&\leq~ \P_{S_1,S_2,\tY_2}\left[\E_{\gamma}[\cd^{\epsilon}_{S_1}(h^{\gamma}_{S_1})] ~\geq~ \cd^{\epsilon}_{\min}(\cG,S_1\cup \tS_2)\right]\\ 
&\quad+ (1+\alpha)~ p + \delta^1_m + \delta^2_{m,p,\alpha},
\end{aligned}
\end{equation}
where $\tY_2=\{\ty_i\}^{m}_{i=1}$ is uniformly selected to be a set of labels that disagrees with $Y_2$ on $p m$ values.
\end{restatable}

\begin{proof}
Let $S_1=\{(x^1_i,y^1_i)\}^{m}_{i=1}$ and $S_2 = \{(x^2_i,y^2_i)\}^{m}_{i=1}$ be two balanced datasets. Let $\epsilon >0$, $p>0$ and $q \geq (1+\alpha)~ p$. Let $\tY_2$ and $\hat{Y}_2$ be a uniformly selected set of labels that disagree with $Y_2$ on $p m$ and $q m$ randomly selected labels (resp.). We denote by $\tS_2$ and $\hat{S}_2$ the relabeling of $S_2$ with the labels in $\tY_2$ and in $\hat{Y}_2$ (resp.). We define four different events,
\begin{equation}
\begin{aligned}
A_1 ~&=~ \{(S_1,S_2,\tY_2) \mid \exists~q \geq (1+\alpha)~p:~\cd^{\epsilon}_{\min}(\cG,S_1\cup \tS_2) > \E_{\hat{Y}_2}[\cd^{\epsilon}_{\min}(\cG,S_1\cup \hat{S}_2)]\} \\
A_2 ~&=~ \{(S_1,S_2) \mid \textnormal{the mistakes of } h^{\gamma}_{S_1} \textnormal{ are not uniform over } S_2\} \\
A_3 ~&=~ \{(S_1,S_2,\tY_2) \mid (S_1,S_2,\tY_2)\notin A_1\cup A_2 \textnormal{ and } \E_{\gamma}[\cd^{\epsilon}_{S_1}(h^{\gamma}_{S_1})] ~<~ \cd^{\epsilon}_{\min}(\cG,S_1\cup \tS_2) \} \\
A_4 ~&=~ \{(S_1,S_2,\tY_2) \mid (S_1,S_2,\tY_2)\notin A_1\cup A_2 \textnormal{ and } \E_{\gamma}[\cd^{\epsilon}_{S_1}(h^{\gamma}_{S_1})] ~\geq~ \cd^{\epsilon}_{\min}(\cG,S_1\cup \tS_2) \} \\
B_1 ~&=~ \{(S_1,S_2,\tY_2) \mid \E_{\gamma}[\cd^{\epsilon}_{S_1}(h^{\gamma}_{S_1})] ~\geq~ \cd^{\epsilon}_{\min}(\cG,S_1\cup \tS_2) \} \\
\end{aligned}
\end{equation}
By the law of total expectation
\begin{equation}
\begin{aligned}
\E_{S_1}\E_{\gamma}[\textnormal{err}_{P}(h^{\gamma}_{S_1})] ~&=~ \E_{S_1,S_2}\E_{\gamma}[\textnormal{err}_{S_2}(h^{\gamma}_{S_1})] \\
~&=~ \sum^{4}_{i=1} \P[A_i] \cdot \E_{S_1,S_2,\tY_2}[\E_{\gamma}[\textnormal{err}_{S_2}(h^{\gamma}_{S_1})] \mid A_i] \\
~&\leq~ \P[A_1] + \P[A_2] + \E_{S_1,S_2,\tY_2}[\textnormal{err}_{S_2}(h^{\gamma}_{S_1}) \mid A_3] + \P[B_1],
\end{aligned}
\end{equation}
where the last inequality follows from $\textnormal{err}_{S_2}(h^{\gamma}_{S_1})\leq 1$, $\P[A_3]\leq 1$ and $A_4 \subset B_1$. 

We would like to upper bound each one of the above terms. First, we notice that since the mistakes of the network are $\delta^1_m$-uniform, $\P[A_2]\leq \delta^1_m$. In addition, by definition $\P[A_1] \leq \delta^2_{m,p,\alpha}$.

As a next step, we upper bound $\E_{S_1,S_2,\tY_2}[\textnormal{err}_{S_2}(h^{\gamma}_{S_1}) \mid A_3]$. Assume that $(S_1,S_2,\tY_2) \in A_3$. Hence, $(S_1,S_2,\tY_2) \notin A_1\cup A_2$. Then, the mistakes of $h^{\gamma}_{S_1}$ over $S_2$ are uniformly distributed (with respect to the selection of $\gamma$). Assume by contradiction that $q_m := \textnormal{err}_{S_2}(h^{\gamma}_{S_1}) > (1+\alpha)~ p$ for some initialization $\gamma$. Then, since the mistakes of $h^{\gamma}_{S_1}$ over $S_2$ are uniformly distributed, $q_m = \textnormal{err}_{S_2}(h^{\gamma}_{S_1}) > (1+\alpha)~ p$ for all initializations $\gamma$. Therefore, we have 
\begin{equation*}
\E_{\hat{Y_2}}[\cd^{\epsilon}_{\min}(\cF,S_1\cup \hat{S}_2)] ~\leq~ \E_{\gamma}[\cd^{\epsilon}_{S_1}(h^{\gamma}_{S_1})] ~<~ \cd^{\epsilon}_{\min}(\cG,S_1\cup \tS_2),
\end{equation*}
where the first inequality follows from the definition of $\cd^{\epsilon}_{\min}(\cF,S_1\cup \hat{S}_2)$ and the second one by the assumption that $(S_1,S_2,\tY_2) \in A_3$. However, this inequality contradicts the fact that $(S_1,S_2,\tY_2) \notin A_1$. Therefore, we conclude that in this case, $q = \E_{\gamma}[\textnormal{err}_{S_2}(h^{\gamma}_{S_1})] \leq (1+\alpha)~p$ and $\E_{S_1,S_2,\tY_2}[\textnormal{err}_{S_2}(h_{S_1}) \mid A_3] \leq (1+\alpha)~p$.
\end{proof}

\begin{restatable}{proposition}{bound_sec}\label{prop:bound2} Let $m \in \mathbb{N}$, $p\in (0,1/2)$, $\alpha \in (0,1)$ and $\epsilon\in (0,1)$. Assume that the error of the learning algorithm is $\delta^1_m$-uniform. Let $S_1,S_2,S^i_1,S^i_2 \sim P_B(m)$ (for $i \in [k]$). Let $\tY^i_2=\{\ty_i\}^{m}_{i=1}$ be a set of labels that disagrees with $Y^i_2$ on uniformly selected $p m$ labels and $\tS^i_2$ is a relabeling of $S_2$ with the labels in $\tY^i_2$. Let $h^{\gamma}_{S_1}$ be the output of the learning algorithm given access to a dataset $S_1$ and initialization $\gamma$. Then, with probability at least $1-\delta$ over the selection of $\{(S^i_1,S^i_2,\tY^i_2)\}^{k}_{i=1}$, we have
\begin{equation*}
\begin{aligned}
\E_{S_1}\E_{\gamma}[\textnormal{err}_P(h^{\gamma}_{S_1})] ~\leq~
&\frac{1}{k}\sum^{k}_{i=1}\bI\left[\E_{\gamma}[\cd^{\epsilon}_{S^i_1}(h^{\gamma}_{S^i_1})] ~\geq~ \cd^{\epsilon}_{\min}(\cG,S^i_1\cup \tS^i_2)\right] \\
& + (1+\alpha)~ p + \delta^1_m + \delta^2_{m,p,\alpha} + \sqrt{\frac{\log(2/\delta)}{2k}}.
\end{aligned}
\end{equation*}
\end{restatable}

\begin{proof}
By Prop.~\ref{prop:bound}, we have
\begin{equation*}
\begin{aligned}
\E_{S_1}\E_{\gamma}[\textnormal{err}_P(h^{\gamma}_{S_1})] ~\leq~& \P_{S_1,S_2,\tY_2}\left[\E_{\gamma}[\cd^{\epsilon}_{S_1}(h^{\gamma}_{S_1})] ~\geq~ \cd^{\epsilon}_{\min}(\cG,S_1\cup \tS_2)\right] \\
&+ (1+\alpha)~p_m + \delta^1_m + \delta^2_{m,p,\alpha}
\end{aligned}
\end{equation*}
We define i.i.d. random variables 
\begin{equation}
V_i ~=~ \bI\left[\E_{\gamma}[\cd^{\epsilon}_{S^i_1}(h^{\gamma}_{S^i_1})] ~\geq~ \cd^{\epsilon}_{\min}(\cG,S^i_1\cup \tS^i_2)\right].
\end{equation}
Therefore, we can rewrite,
\begin{equation}
\P_{S_1,S_2,\tY_2}\left[\E_{\gamma}[\cd^{\epsilon}_{S_1}(h^{\gamma}_{S_1})] ~\geq~ \cd^{\epsilon}_{\min}(\cG,S_1\cup \tS_2) \right] ~=~ \E[V_1]
\end{equation}
By Hoeffding's inequality,
\begin{equation}
\Pr\left[\left\vert k^{-1}\sum^{k}_{i=1}V_i - \E[V_1] \right\vert ~\geq~ \epsilon\right] ~\leq~ 2\exp(-2k\epsilon^2).
\end{equation}
By choosing $\epsilon = \sqrt{\log(1/2\delta)/2k}$, we obtain that with probability at least $1-\delta$, we have
\begin{equation}
\E[V_1] ~\leq~ \frac{1}{k}\sum^{k}_{i=1} V_i + \sqrt{\log(1/2\delta)/2k}.
\end{equation}
When combined with Prop.~\ref{prop:bound}, we obtain the desired bound.
\end{proof}

\bibliographystyle{abbrvnat}
\bibliography{neurips_2022}